\documentclass[11pt]{article}

\usepackage{hyperref,graphicx,amsmath,amsfonts,amssymb,bm,url,breakurl,epsfig,epsf,color,fullpage,MnSymbol,mathbbol,fmtcount,semtrans,cite,caption,subcaption,multirow,comment}
\usepackage[noend]{algpseudocode}

\usepackage{amssymb}
\usepackage{enumerate}
\usepackage[utf8]{inputenc} 
\usepackage[T1]{fontenc}    
\usepackage{url}            
\usepackage{booktabs}       
\usepackage{amsfonts}       
\usepackage{nicefrac}       
\usepackage{microtype}      

\makeatletter
\providecommand*{\boxast}{%
  \mathbin{
    \mathpalette\@boxit{*}%
  }%
}
\newcommand*{\@boxit}[2]{%
  \sbox0{$\m@th#1\Box$}%
  \ifx#1\displaystyle \ht0=\dimexpr\ht0+.05ex\relax \fi
  \ifx#1\textstyle \ht0=\dimexpr\ht0+.05ex\relax \fi
  \ifx#1\scriptstyle \ht0=\dimexpr\ht0+.04ex\relax \fi
  \ifx#1\scriptscriptstyle \ht0=\dimexpr\ht0+.065ex\relax \fi
  \sbox2{$#1\vcenter{}$}
  \rlap{%
    \hbox to \wd0{%
      \hfill
      \raisebox{%
        \dimexpr.5\dimexpr\ht0+\dp0\relax-\ht2\relax
      }{$\m@th#1#2$}%
      \hfill
    }%
  }%
  \Box
}
\makeatother

  \makeatletter
\def\BState{\State\hskip-\ALG@thistlm}
\makeatother

  \usepackage{mathtools}

\usepackage{titlesec}

\usepackage{tikz}
\usepackage{pgfplots}
\usetikzlibrary{pgfplots.groupplots}

\titleformat{\paragraph}
{\normalfont\normalsize\bfseries}{\theparagraph}{1em}{}
\titlespacing*{\paragraph}
{0pt}{3.25ex plus 1ex minus .2ex}{1.5ex plus .2ex}

\usepackage{movie15}

\usepackage{cite}
\usepackage{caption}
\usepackage[bottom,hang,flushmargin]{footmisc} 

\newcommand{\tsn}[1]{{\left\vert\kern-0.25ex\left\vert\kern-0.25ex\left\vert #1 
    \right\vert\kern-0.25ex\right\vert\kern-0.25ex\right\vert}}

\definecolor{darkred}{RGB}{150,0,0}
\definecolor{darkgreen}{RGB}{0,150,0}
\definecolor{darkblue}{RGB}{0,0,200}
\hypersetup{colorlinks=true, linkcolor=darkred, citecolor=darkgreen, urlcolor=darkblue}

\newtheorem{theorem}{Theorem}[section]

\newtheorem{lemma}[theorem]{Lemma}

\newcommand{\frs}{f_{\text{res}}}
\newcommand{\crs}{c_{\text{res}}}
\newcommand{\cph}{c_{\phi}}
\newcommand{\bp}{\beta}
\newcommand{\bn}{\alpha}

\newcommand{\beq}{\begin{equation}}

\newcommand{\eeq}{\end{equation}}

\newcommand{\Dc}{{\cal{D}}}

\newcommand{\tn}[1]{\|{#1}\|_{2}}

\newcommand{\bteta}{\boldsymbol{\theta}}

\newcommand{\w}{\vct{w}}

\newcommand{\opnorm}[1]{\left\|#1\right\|}
\newcommand{\fronorm}[1]{\left\|#1\right\|_{F}}

\newcommand{\twonorm}[1]{\left\|#1\right\|_{2}}

\newcommand{\infnorm}[1]{\left\|#1\right\|_{\infty}}

\newcommand{\abs}[1]{\left|#1\right|}

\newcommand{\x}{\vct{x}}

\newcommand{\y}{\vct{y}}

\newcommand{\W}{\mtx{W}}

\definecolor{emmanuel}{RGB}{255,127,0}

\newcommand{\R}{\mathbb{R}}

\newcommand{\E}{\operatorname{\mathbb{E}}}

\newcommand{\vct}[1]{\bm{#1}}
\newcommand{\mtx}[1]{\bm{#1}}

\newcommand{\X}{{\mtx{X}}}

\numberwithin{equation}{section} 

\def \endprf{\hfill {\vrule height6pt width6pt depth0pt}\medskip}
\newenvironment{proof}{\noindent {\bf Proof} }{\endprf\par}

\newcommand\remove[1]{}
\usepackage{multirow}

\usepackage{relsize}

\usepackage{algorithm}
\usepackage{algpseudocode}

\title{Stronger Convergence Results for Deep Residual Networks: Network Width Scales Linearly with Training Data Size}
\author{Talha Cihad Gulcu \thanks{{ASELSAN Inc. e-mail:tcgulcu@gmail.com}}}
\begin{document}
\maketitle
\begin{abstract}
Deep neural networks are highly expressive machine learning models with the ability to interpolate arbitrary datasets. Deep nets are typically optimized via first-order methods and the optimization process crucially depends on the characteristics of the network as well as the dataset. This work sheds light on the relation between the network size and the properties of the dataset with an emphasis on deep residual networks (ResNets). Our contribution is that if the network Jacobian is full rank, gradient descent for the quadratic loss and smooth activation converges to the global minima even if the network width $m$ of the ResNet scales linearly with the sample size $n$, and independently from the network depth. 
To the best of our knowledge, this is the first work which provides
a theoretical guarantee for the convergence of neural networks in the $m=\Omega(n)$ regime. 

\end{abstract}

\section{Introduction}

Deep neural networks have gained remarkable success over a large variety of applications, including computer vision\cite{he2016deep}, natural language processing\cite{young2018recent}, speech recognition
\cite{hinton2012deep} and Go games \cite{silver2016mastering}. But the reason why deep networks perform well over various tasks is still not exactly understood. The optimization performance of deep networks is one of the subjects which requires an involved theoretical study, given that
gradient descent can achieve zero training loss even for random labels
\cite{zhang2016understanding}, and the loss of deep networks is highly non-convex. There are different lines of works investigating the optimization of deep networks from different perspectives. For example, a large number of works consider the optimization landscape corresponding to different activation functions
\cite{Soudry2016no,zhou2017landscape,hardt2016identity,
choromanska2015loss,nguyen2017loss,soltanolkotabi2018theoretical}, whereas some others 
\cite{soltanolkotabi2017learning,du2017gradient,brutzkus2017globally,tian2017analytical}
ensure global convergence by imposing some restrictions on the input distribution. 

In the recent years, there has been considerably many papers providing convergence guarantees for over-parameterized two-layer and deep networks. It is shown in \cite{li2018learning} that gradient descent can find the near-global minima of a single hidden layer network in polynomial time with respect to the accuracy and sample size. A similar observation has been made by \cite{du2018gradient2} in regards to the single hidden layer networks where the network width
has to taken as $m=\Omega(n^6).$ Later on, the convergence problem
for over-parameterized multi-layer networks is considered by
\cite{allen2018convergence,zou2018stochastic,du2018gradient}.
It is proved in \cite{allen2018convergence} that convergence to near-global minima is possible
when the network width $m$ grows polynomially with the sample size $n$
and the network depth $H$. The authors claim this result is valid for
convolution neural networks(CNNs) and deep residual networks(ResNets) as well. \cite{zou2018stochastic} shows that convergence time is 
polynomial in $H$ and $n$ for deep ReLU networks when $m=\text{poly}(n,H).$ The paper \cite{du2018gradient} proves that the condition $m=\Omega(\text{poly}(n) 2^{O(H)})$ suffices for the gradient descent to converge to zero training loss for fully-connected feedforward network, and $m=\text{poly}(n,H)$ is enough to have a similar convergence for ResNets and convolutional ResNets.

A dataset dependent convergence rate is specified in \cite{arora2019fine},
where the condition $m=\Omega(n^7)$ ensures a linear convergence for
the single hidden layer network, and the convergence rate is related to the Gram matrix ${\mathbf H}^{\infty}$ of the dataset. 
Another single hidden layer network analysis appears
in \cite{oymak2019towards}, which shows that a moderate overparameterization of the form $m=O(n^2)$
is sufficient to have a global convergence when the network has single hidden layer. In the paper \cite{zhang2019training}, the authors conclude $m$ can grow independently of $H$ and the linear convergence can still be maintained for the ResNet when $m=\Omega(n^{24}).$ Later on, the same authors prove some convergence results (where $m$ depends polynomially on $H$ this time) 
for the case when the scaling factor of the ResNet is taken as 
$\tau=O( 1/\sqrt{H})$ instead of $\tau=O(1/H).$
Convergence for the adversarial training is considered in 
\cite{gao2019convergence},
where it is shown that it is possible to find the optimal set of weight matrices inside a radius of $B$ and within an error tolerance $\epsilon$ for the fully connected feedforward network if 
$m=\Omega(B^4 \,n\, 2^{O(H)}/\epsilon^2).$
A list of papers providing convergence guarantees along with
the conditions and network types for which those convergences are valid is given by Table~\ref{table1}.

\begin{table}[h]
\centering
\begin{tabular}{ |c|c|c|c|c| } 
 \hline
Paper & Network Type & \shortstack{Activation\\ Function} & Loss Function & \shortstack{Network \\Width $m$}\\ 
\hline
\cite{li2018learning} & Single hidden layer & ReLU & Cross Entropy & $\Omega(\text{poly}(n))$\\
\hline
\cite{du2018gradient2} & Single hidden layer & ReLU & Quadratic & $\Omega(n^6)$ \\ 
\hline
\cite{allen2018convergence} & \shortstack{Multi-layer,\\ CNN, ResNet}
& ReLU & Arbitrary & $\Omega(\text{poly}(n,H))$\\
\hline
\cite{zou2018stochastic} & \shortstack{Multi-layer\\ feedforward} &
ReLU & Smooth & $\Omega(\text{poly}(n,H))$\\
\hline
\cite{du2018gradient} & \shortstack{Multi-layer\\ feedforward} 
& Smooth & Quadratic & $\Omega(\text{poly}(n)2^{O(H)})$\\
\hline
\cite{du2018gradient} & CNN, ResNet & Smooth&
Quadratic & $\Omega(\text{poly}(n,H))$\\
\hline
\cite{arora2019fine} & Single hidden layer & ReLU & Quadratic &
$\Omega(n^7)$\\
\hline
\cite{oymak2019towards} & Single hidden layer & Smooth &
Quadratic & $\Omega(n^2)$\\
\hline
\cite{oymak2019towards} & Single hidden layer & ReLU &
Quadratic & $\Omega(n^4)$\\
\hline
\cite{zhang2019training} & ResNet & ReLU & Quadratic &
$\max(H,\Omega(n^{24}))$\\
\hline
\cite{zhang2019training} & \shortstack{ResNet with\\ scaling $\frac{1}{\sqrt{H}}$}
& ReLU & Quadratic & $\Omega(n^{24}H^7)$\\
\hline
\cite{gao2019convergence} & \shortstack{Multi-layer\\ feedforward} & Smooth & Smooth
& $\Omega(B^4 \,n\, 2^{O(H)})$\\
\hline
This paper & ResNet & Smooth & Quadratic & $\Omega(n)$\\
\hline
\end{tabular}
\caption{A summary of the some recent convergence results for the 
overparameterized neural networks.}
\label{table1}
\end{table}

In this paper, we consider the ResNet model with the quadratic loss function. We assume the activation function has bounded first and
second derivatives. 
Our convergence proof strategy relies on the meta theorem given by
\cite{Oymak:2018aa}, which basically states that if the Lipschitz constant and the maximum eigenvalue of the Jacobian is bounded over a ball centered at the random initialization, and if the minimum eigenvalue of the Jacobian is bounded away from zero in the same ball, then linear convergence rate is achievable for the quadratic loss function.
After an involved analysis, it turns out that those conditions of the meta theorem can be met if $m=\Omega(n)$, i.e., the network width scales linearly with the sample size.
A convergence result associated with the linear growth of network width is not available in the literature even for single-hidden layer networks. Moreover, this is the first global minima finding result for the deep networks with the properties
\begin{enumerate}[i.]
\item The growth rate of network width does not depend on the
network depth,
\item The dependence of the network width to the sample size is linear,
\end{enumerate}
to the best of our knowledge.

The rest of the paper is organized as follows. We introduce the problem setup in Section~\ref{problem_setup}. The derivation of Jacobian related parameters playing a crucial role in the linear convergence rate is performed in Section~\ref{jacob_analysis}. An upper bound for the error term for the random initialization is found in Section~\ref{misfit_section}. The calculations showing how much the Jacobian parameters change over a given ball centered at random initialization is presented in Section~\ref{param_variation}.
Our main result and a list of some future work are stated in Section~\ref{concl}.

\section{Problem Setup}
\label{problem_setup}
In this work, we consider the empirical loss minimization 
problem for the quadratic loss function 
\begin{equation}
\min_{\boldsymbol{\theta}}L(\boldsymbol{\theta})=
\frac{1}{2}\sum_{i=1}^n 
\left(f({\mathbf x}_i,\boldsymbol{\theta})-y_i\right)^2 \label{quadratic_loss}   
\end{equation}
where $\{{\mathbf x}_i\}_{i=1}^n$ and $\{y_i\}_{i=1}^n$
refer to the training inputs and labels, respectively,
$\boldsymbol{\theta}$ is the parameter set to be optimized, and $f$
is the output of the neural network. We focus on
deep residual networks, which forms
a particular class of neural networks.

For the classical fully-connected neural network, the
network output $f$ is defined recursively as follows.
Denoting the input as ${\mathbf x}\in{\mathbb R}^d$, 
the activation function as $\phi(\cdot)$,
the first weight matrix as ${\mathbf W}^{(1)}\in{\mathbb R}^{m\times d}$, the weight matrix associated with the $h-$th layer for
$2\leq h\leq H$ as ${\mathbf W}^{(h)}\in{\mathbb R}^{m\times m},$
and the output layer weight vector as ${\mathbf a}\in{\mathbb R}^m,$
we express $f({\mathbf x},\theta)$ as
\begin{align}
{\mathbf x}^{(h)}&=\sqrt{\frac{\cph}{m}}
\phi\left({\mathbf W}^{(h)}{\mathbf x}^{(h-1)}\right),
\,\,1\leq h\leq H\nonumber\\
f({\mathbf x},\boldsymbol{\theta})&={\mathbf a}^T{\mathbf x}^{(H)}\label{fully_conn}
\end{align}
where we take ${\mathbf x}^{(0)}={\mathbf x}$, and $\cph=1/{\mathbb E}_{x\sim{\mathcal N}(0,1)}[\phi(x)]^2$
is the normalization factor.

The network output for the deep residual network (ResNet) involves
a similar recursion, with the difference of an additional term which
makes it possible to skip some of the layers as the input propagates.
The ResNet output model is defined formally in the next section.

\subsection{ResNet model}
Using the same notation as the fully-connected neural network output
\eqref{fully_conn}, the ResNet model is defined recursively as
\begin{align}
{\mathbf x}^{(1)}&=\sqrt{\frac{\cph}{m}}\phi\left({\mathbf W}^{(1)}{\mathbf x}\right),\nonumber\\ 
{\mathbf x}^{(h)}&={\mathbf x}^{(h-1)}+\frac{\crs}{H\sqrt{m}}
\phi\left({\mathbf W}^{(h)}{\mathbf x}^{(h-1)}\right),\,\, 
2\leq h\leq H,
\nonumber\\
\frs({\mathbf x},\boldsymbol{\theta})&={\mathbf a}^T{\mathbf x}^{(H)}\label{Res_Net_eq}
\end{align}
where $\crs$ is a positive constant less than $1.$ 
We assume the activation
function $\phi$ satisfies $|\phi'(z)|\leq B$ and 
$|\phi''(z)|\leq M$ for some constants $B$ and $M$.
We observe from \eqref{Res_Net_eq} that there is an extra term ${\mathbf x}^{(h-1)}$
being added to the $h$-th layer output ${\mathbf x}^{(h)}$. 
Because of this term, for each
subset $S=\{h_1,\dots,h_k\}$ of the set $\Omega\triangleq\{2,\dots,H\}$, there is a connected network $1\to h_1\to\dots\to h_k$ contributing to the last layer output ${\mathbf x}^{(H)}.$ More precisely, we have
\begin{align}
|{\mathbf x}^{(H)}|&\leq
\sqrt{\frac{\cph}{m}}\left|\phi\left({\mathbf W}^{(1)}{\mathbf x}\right)\right|\nonumber\\
&\quad+
\mathlarger{\mathlarger{\sum}}_
{\substack{\{h_1,\dots,h_k\}\\
\subseteq\{2,\dots,H\}}}\,\, 
\frac{\crs}{H\sqrt{m}}\Bigg|\phi\Bigg(\frac{\crs}{H\sqrt{m}}{\mathbf W}^{(h_k)}
\phi\Bigg(\frac{\crs}{H\sqrt{m}}{\mathbf W}^{(h_{k-1})}\dots
\phi\Bigg(\sqrt{\frac{\cph}{m}} {\mathbf W}^{(h_{1})}
\phi({\mathbf W}^{(1)}{\mathbf x})\Bigg)\dots\Bigg)\Bigg)\Bigg|
\label{ResNet_uppbd}
\end{align}
which can be proven easily via induction. The inequality \eqref{ResNet_uppbd} holds trivially for $H=1$, and the inductive
step follows from the inequality 
$|\phi({\mathbf a}+{\mathbf b})|\leq |\phi({\mathbf a})|+
|\phi({\mathbf b})|$ which holds for commonly used activation
functions such as ReLU and softplus. 

\subsection{Global Convergence Overview}

Given a neural network with the label output $f$, parameter set 
$\boldsymbol{\theta}$, and loss function $L$, the gradient descent
iterations of $\boldsymbol{\theta}$ can be written as
\begin{equation}
\boldsymbol{\theta}(k)=\boldsymbol{\theta}(k-1)-\eta 
\frac{\partial L}{\partial\boldsymbol{\theta}}\bigg|_{\boldsymbol{\theta}=\boldsymbol{\theta}(k-1)}\label{grad_des}
\end{equation}
for $k=1,2,\dots,$ where $\eta$ is the step size. When the loss function is quadratic as in \eqref{quadratic_loss}, the gradient descent update \eqref{grad_des} takes the form
\begin{align}
\boldsymbol{\theta}(k)&=\boldsymbol{\theta}(k-1)-\eta\,\,
\mathlarger{\sum}_{i=1}^n \hspace{.1in}\frac{\partial f({\mathbf x}_i,\boldsymbol{\theta})}{\partial \boldsymbol{\theta}}\bigg|_{\boldsymbol{\theta}=
\boldsymbol{\theta}(k-1)} \hspace{-.2in}\cdot
\left(f({\mathbf x}_i,\boldsymbol{\theta}(k-1))-y_i\right)
\nonumber\\
&=\boldsymbol{\theta}(k-1)-\eta\left[\frac{\partial f({\mathbf x}_1,\boldsymbol{\theta})}{\partial \boldsymbol{\theta}}
\Bigg|_{\boldsymbol{\theta}=\boldsymbol{\theta}(k-1)} \hspace{-.2in}\dots\,\,\,
\frac{\partial f({\mathbf x}_n,\boldsymbol{\theta})}{\partial \boldsymbol{\theta}}\Bigg|_{\boldsymbol{\theta}=\boldsymbol{\theta}(k-1)}\right] .
\begin{bmatrix}
f({\mathbf x}_1,\boldsymbol{\theta}(k-1))-y_1 \\
\vdots\\
f({\mathbf x}_n,\boldsymbol{\theta}(k-1))-y_n
\end{bmatrix}
\label{quad_update}
\end{align}
Using \eqref{quad_update} and first order Taylor approximation,
we get

\begin{align}
\begin{bmatrix}
f({\mathbf x}_1,\boldsymbol{\theta}(k)) \\
\vdots\\
f({\mathbf x}_n,\boldsymbol{\theta}(k))
\end{bmatrix}
\approx
\begin{bmatrix}
f({\mathbf x}_1,\boldsymbol{\theta}(k-1)) \\
\vdots\\
f({\mathbf x}_n,\boldsymbol{\theta}(k-1))
\end{bmatrix}
&-\eta \begin{bmatrix}
\frac{\partial f({\mathbf x}_1,\boldsymbol{\theta})}{\partial \boldsymbol{\theta}}
\big|_{\boldsymbol{\theta}=\boldsymbol{\theta}(k-1)} \\
\vdots\\
\frac{\partial f({\mathbf x}_n,\boldsymbol{\theta})}{\partial \boldsymbol{\theta}}
\big|_{\boldsymbol{\theta}=\boldsymbol{\theta}(k-1)}
\end{bmatrix} \cdot \left[\frac{\partial f({\mathbf x}_1,\boldsymbol{\theta})}{\partial \boldsymbol{\theta}}
\Bigg|_{\boldsymbol{\theta}=\boldsymbol{\theta}(k-1)} \hspace{-.2in}\dots\,\,\,
\frac{\partial f({\mathbf x}_n,\boldsymbol{\theta})}{\partial \boldsymbol{\theta}}\Bigg|_{\boldsymbol{\theta}=\boldsymbol{\theta}(k-1)}\right] \nonumber\\&\quad\cdot
\begin{bmatrix}
f({\mathbf x}_1,\boldsymbol{\theta}(k-1))-y_1 \\
\vdots\\
f({\mathbf x}_n,\boldsymbol{\theta}(k-1))-y_n
\end{bmatrix}\nonumber
\end{align}
which can be written in more compactly as
\begin{align}
\begin{bmatrix}
f({\mathbf x}_1,\boldsymbol{\theta}(k)) \\
\vdots\\
f({\mathbf x}_n,\boldsymbol{\theta}(k))
\end{bmatrix}
\approx
\begin{bmatrix}
f({\mathbf x}_1,\boldsymbol{\theta}(k-1)) \\
\vdots\\
f({\mathbf x}_n,\boldsymbol{\theta}(k-1))
\end{bmatrix}
-\eta [{\mathcal J}(\boldsymbol{\theta})
{\mathcal J}^T(\boldsymbol{\theta})]\bigg|_{\boldsymbol{\theta}=\boldsymbol{\theta}(k-1)}
\begin{bmatrix}
f({\mathbf x}_1,\boldsymbol{\theta}(k-1))-y_1 \\
\vdots\\
f({\mathbf x}_n,\boldsymbol{\theta}(k-1))-y_n
\end{bmatrix}\label{taylor}
\end{align}
where ${\mathcal J}(\boldsymbol{\theta})\triangleq \left[\frac{\partial f({\mathbf x}_1,\boldsymbol{\theta})}{\partial \boldsymbol{\theta}}
\dots
\frac{\partial f({\mathbf x}_n,\boldsymbol{\theta})}{\partial \boldsymbol{\theta}}\right]^T$ is the neural network Jacobian matrix, and ${\mathcal J}(\boldsymbol{\theta})
{\mathcal J}^T(\boldsymbol{\theta})$ is referred as
neural tangent kernel(NTK)\cite{jacot2018neural,arora2019exact}.
Eq. \eqref{taylor} can be trivially transformed to
\begin{align}
\begin{bmatrix}
y_1-f({\mathbf x}_1,\boldsymbol{\theta}(k)) \\
\vdots\\
y_n-f({\mathbf x}_n,\boldsymbol{\theta}(k))
\end{bmatrix}
\approx \left[{\mathbf I}-\eta {\mathcal J}(\boldsymbol{\theta}(k-1))
{\mathcal J}^T(\boldsymbol{\theta}(k-1))\right]
\begin{bmatrix}
y_1-f({\mathbf x}_1,\boldsymbol{\theta}(k-1)) \\
\vdots\\
y_n-f({\mathbf x}_n,\boldsymbol{\theta}(k-1))
\end{bmatrix}\label{taylor2}
\end{align}

We observe from \eqref{taylor2} that the convergence rate of the gradient descent iteration is related to the minimum singular value
of the Jacobian ${\mathcal J}(\boldsymbol{\theta}).$ If we make sure
that the parameter set $\boldsymbol{\theta}$ stays inside a certain 
region during the gradient descent updates and the minimum singular value of ${\mathcal J}(\boldsymbol{\theta})$ is greater than a certain threshold in this region, then the global convergence of the training
process should follow. Arguing the global convergence of the training in this way is encountered in numerous papers including \cite{du2018gradient2,du2018gradient,arora2019fine,oymak2019towards}.

\subsection{Meta convergence theorem}
To prove our convergence results, we will utilize a result from \cite{Oymak:2018aa} stated below.
\begin{theorem}\label{GDthm} Consider a nonlinear least-squares optimization problem given by $\underset{\vct{\theta}\in\R^p}{\min}\text{ }\mathcal{L}(\vct{\theta}):=\frac{1}{2}\twonorm{f(\vct{\theta})-\vct{y}}^2$ with $f:\R^p\mapsto \R^n$ and $\vct{y}\in\R^n$. Given $L,\alpha>0$, consider an Euclidean ball $\Dc$ of radius $R=\frac{4\tn{f(\bteta_0)-\y}}{\alpha}$ around the initial point $\bteta_0$. Suppose the Jacobian associated with $f$ obeys
\begin{align}
\label{bndspect}
\alpha\leq
\phi_{\min}\left(\mathcal{J}(\vct{\theta})\right)\leq
\|\mathcal{J}(\vct{\theta})\|\le \bp.
\end{align}
at all $\bteta\in\Dc$. Furthermore, suppose $\opnorm{\mathcal{J}(\vct{\theta}_2)-\mathcal{J}(\vct{\theta}_1)}\le L\twonorm{\vct{\theta}_2-\vct{\theta}_1}$ holds for any $\vct{\theta}_1,\vct{\theta}_2\in\Dc$ 
\remove{
and assume initialization obeys\vspace{-6pt}
\begin{align}
\label{corassumption}
\frac{\bn^2}{4L}\ge \twonorm{f(\vct{\theta}_0)-\vct{y}}.
\end{align}
}
Then, setting $\eta\leq 
\frac{1}{2 \bp^2}\cdot\min\left(1,
\frac{\alpha^2}{L\twonorm{f(\vct{\theta}_0)-\vct{y}}}\right)$ and running gradient descent updates given by $\vct{\theta}_{\tau+1}=\vct{\theta}_\tau-\eta\nabla\mathcal{L}(\vct{\theta}_\tau)$ starting from $\vct{\theta}_0$, all iterates obey
\begin{align}
\twonorm{f(\vct{\theta}_\tau)-\vct{y}}^2\le&\left(1-\frac{\eta\bn^2}{2}\right)^\tau\twonorm{f(\vct{\theta}_0)-\vct{y}}^2,\label{err}\\
\frac{\bn}{4}\twonorm{\vct{\theta}_\tau-\vct{\theta}_0}+\twonorm{f(\vct{\theta}_\tau)-\vct{y}}\le&\twonorm{f(\vct{\theta}_0)-\vct{y}}.\label{close}
\end{align}
\end{theorem}

To apply this theorem it suffices to prove the conditions above for proper choices of $\bn, \bp,$ and $L$. Therefore, our aim is to compute the values those variable take for the ResNet. As a preliminary remark, we present the results proved in \cite{oymak2019towards} for the one-hidden layer neural network model, defined as $f({\mathbf x},{\mathbf W})={\mathbf v}^T\phi({\mathbf W}{\mathbf x}),$ for
${\mathbf v}\in{\mathbb R}^k, {\mathbf W}\in{\mathbb R}^{k\times d},
{\mathbf x}\in{\mathbb R}^d.$

\begin{lemma}[Spectral norm of the Jacobian]\label{spectJ} Consider a one-hidden layer neural network model of the form $\vct{x}\mapsto \vct{v}^T\phi\left(\W\x\right)$ where the activation $\phi$ has bounded derivatives obeying $\abs{\phi'(z)}\le B$. Also assume we have $n$ data points $\vct{x}_1, \vct{x}_2,\ldots,\vct{x}_n\in\R^d$ aggregated as the rows of a matrix $\X\in\R^{n\times d}$. Then the Jacobian matrix with respect to the input-to-hidden weights obeys 
\begin{align*}
\opnorm{\mathcal{J}(\mtx{W})}\le \sqrt{k}B\infnorm{\vct{v}}\opnorm{\X}.
\end{align*}
\end{lemma}

\begin{lemma}[Minimum eigenvalue of the Jacobian at initialization]\label{minspectJ} Consider a one-hidden layer neural network model of the form $\vct{x}\mapsto \vct{v}^T\phi\left(\W\x\right)$ where the activation $\phi$ has bounded derivatives obeying $\abs{\phi'(z)}\le B$. Also assume we have $n$ data points $\vct{x}_1, \vct{x}_2,\ldots,\vct{x}_n\in\R^d$ with unit euclidean norm ($\twonorm{\vct{x}_i}=1$) aggregated as the rows of a matrix $\X\in\R^{n\times d}$. Then, as long as
\begin{align*}
\frac{\twonorm{\vct{v}}}{\infnorm{\vct{v}}}\ge \sqrt{20\log n}\frac{\opnorm{\X}}{\sqrt{\lambda(\X)}}B,
\end{align*}
the Jacobian matrix at a random point $\mtx{W}_0\in\R^{k\times d}$ with i.i.d.~$\mathcal{N}(0,1)$ entries obeys 
\begin{align*}
\phi_{\min}\left(\mathcal{J}(\W_0)\right)\ge \frac{1}{\sqrt{2}}\twonorm{\vct{v}}\sqrt{\lambda(\X)},
\end{align*}
with probability at least $1-1/n$, where
$\lambda(\X):=\phi_{\min}\left(\mtx{\Sigma}(\X)\right)$ for the matrix
$$\mtx{\Sigma}(\X):=\E_{\w\sim\mathcal{N}(\vct{0},\mtx{I}_d)}\Big[\left(\phi'\left(\X\w\right)\phi'\left(\X\w\right)^T\right)\odot\left(\X\X^T\right)\Big].$$
\end{lemma}

\begin{lemma}[Jacobian Lipschitzness]\label{JLlem} Consider a one-hidden layer neural network model of the form $\vct{x}\mapsto \vct{v}^T\phi\left(\W\x\right)$ where the activation $\phi$ has bounded second order derivatives obeying $\abs{\phi''(z)}\le M$. Also assume we have $n$ data points $\vct{x}_1, \vct{x}_2,\ldots,\vct{x}_n\in\R^d$ with unit euclidean norm ($\twonorm{\vct{x}_i}=1$) aggregated as the rows of a matrix $\X\in\R^{n\times d}$. Then the Jacobian mapping with respect to the input-to-hidden weights obeys
\begin{align*}
\opnorm{\mathcal{J}(\widetilde{\mtx{W}})-\mathcal{J}(\mtx{W})}\le M\infnorm{\vct{v}}\opnorm{\mtx{X}}\fronorm{\widetilde{\mtx{W}}-\mtx{W}}\quad\text{for all}\quad \widetilde{\W},\W\in\R^{k\times d}.
\end{align*}
\end{lemma}

We will makes use of those lemmas and their proofs to evaluate the
parameters $\alpha,\beta, L$ for the ResNet. But we first need to analyze the Jacobian of the ResNet and express it in a convenient way.
This would be the subject of next section.

\section{The Jacobian Analysis of the ResNet Model}
\label{jacob_analysis}

We see from \eqref{Res_Net_eq} that 
the ResNet output $\frs$ depends on the weight matrices
${\mathbf W}^{(1)},\dots,{\mathbf W}^{(H)}$ and ${\mathbf a}$.
Here we take ${\mathbf a}$ to be a constant vector to which gradient
descent iterations do not apply, so the parameter set 
$\boldsymbol{\theta}$ of our ResNet model is taken to be $\boldsymbol{\theta}=[{\mathbf W}^{(1)} \dots {\mathbf W}^{(H)}].$ 
Hence the matrix product ${\mathcal J}(\boldsymbol{\theta}){\mathcal J}^T(\boldsymbol{\theta})\in {\mathbb R}^{n\times n}$ 
can be expressed as
\begin{align}
[{\mathcal J}(\boldsymbol{\theta}){\mathcal J}^T(\boldsymbol{\theta})]_{ij}&=\left\langle\frac{\partial \frs({\mathbf x}_i,\boldsymbol{\theta})}{\partial\boldsymbol{\theta}},\frac{\partial \frs({\mathbf x}_j,\boldsymbol{\theta})}{\partial\boldsymbol{\theta}}
\right\rangle\nonumber\\
&=\sum_{h=1}^H \left\langle\frac{\partial \frs({\mathbf x}_i,\boldsymbol{\theta})}{\partial{\mathbf W}^{(h)}},
\frac{\partial \frs({\mathbf x}_j,\boldsymbol{\theta})}{\partial{\mathbf W}^{(h)}}
\right\rangle
\nonumber\\&\triangleq
\sum_{h=1}^{H} {\mathbf G}_{ij}^{(h)},\label{posdef}
\end{align}
similarly to the matrix decomposition appearing in \cite{du2018gradient}. 
To compute $\frac{\partial \frs({\mathbf x},\boldsymbol{\theta})}{\partial{\mathbf W}^{(h)}}$ or ${\mathbf G}_{ij}^{(h)}$ for a given $h$, we consider
the sequence of partials $\frac{\partial{\mathbf x}^{(l)}}{\partial{\mathbf W}^{(1)}}\in {\mathbb R}^{m\times md}, \,l=1,\dots,H$ along with $\frac{\partial{\mathbf x}^{(l)}}{\partial{\mathbf W}^{(h)}}\in {\mathbb R}^{m\times m^2}, \,\,l,h=1,\dots,H$, and observe

\begin{align}
\frac{\partial{\mathbf x}^{(l)}}{\partial{\mathbf W}^{(h)}}
=\begin{cases}
0, &\text{if}\quad l<h\\
\sqrt{\frac{\cph}{m}}\,\,
\text{diag} [\phi'({\mathbf W}^{(1)} {\mathbf x})]
\begin{bmatrix}
{\mathbf x}^T & \boldsymbol{0} & \boldsymbol{0}\\
\vdots &\ddots & \vdots\\
\boldsymbol{0} & \boldsymbol{0} &{\mathbf x}^T \\
\end{bmatrix},
&\text{if}\quad l=h=1\\
\frac{\crs}{H\sqrt{m}}\,\,\text{diag} [\phi'({\mathbf W}^{(h)} {\mathbf x}^{(h-1)})]
\begin{bmatrix}
{{\mathbf x}^{(h-1)}}^T & \boldsymbol{0} & \boldsymbol{0}\\
\vdots &\ddots & \vdots\\
\boldsymbol{0} & \boldsymbol{0} &{{\mathbf x}^{(h-1)}}^T \\
\end{bmatrix},
&\text{if}\quad l=h>1\\
\left[{\mathbf I}_m+\frac{\crs}{H\sqrt{m}}
\text{diag} [\phi'({\mathbf W}^{(l)} {\mathbf x}^{(l-1)})]
{\mathbf W}^{(l)}
\right]\frac{\partial{\mathbf x}^{(l-1)}}{\partial{\mathbf W}^{(h)}}
, &\text{if}\quad l>h.
\end{cases}
\label{meta_partial}
\end{align}
It follows from \eqref{meta_partial} that the final layer partial
$\frac{\partial{\mathbf x}^{(H)}}{\partial{\mathbf W}^{(h)}}, h=1,\dots, H$ can be derived as
\begin{align}
\frac{\partial{\mathbf x}^{(H)}}{\partial{\mathbf W}^{(h)}}=
\begin{cases}
\mathlarger{\mathlarger{\mathlarger{\prod}}}_{l=2}^{H} \left[{\mathbf I}_m+\frac{\crs}{H\sqrt{m}}
\text{diag} [\phi'({\mathbf W}^{(l)} {\mathbf x}^{(l-1)})]
{\mathbf W}^{(l)}
\right]\\ \quad\cdot \sqrt{\frac{\cph}{m}}\,\,
\text{diag} [\phi'({\mathbf W}^{(1)} {\mathbf x})]
\begin{bmatrix}
{\mathbf x}^T & \boldsymbol{0} & \boldsymbol{0}\\
\vdots &\ddots & \vdots\\
\boldsymbol{0} & \boldsymbol{0} &{\mathbf x}^T \\
\end{bmatrix}, \quad\text{if}\quad h=1,\\
\mathlarger{\mathlarger{\mathlarger{\prod}}}_{l=h+1}^{H} \left[{\mathbf I}_m+\frac{\crs}{H\sqrt{m}}
\text{diag} [\phi'({\mathbf W}^{(l)} {\mathbf x}^{(l-1)})]
{\mathbf W}^{(l)}
\right]\\ \quad\cdot\frac{\crs}{H\sqrt{m}}\,\,\text{diag} [\phi'({\mathbf  W}^{(h)} {\mathbf x}^{(h-1)})]
\begin{bmatrix}
{{\mathbf x}^{(h-1)}}^T & \boldsymbol{0} & \boldsymbol{0}\\
\vdots &\ddots & \vdots\\
\boldsymbol{0} & \boldsymbol{0} &{{\mathbf x}^{(h-1)}}^T \\
\end{bmatrix}, \quad\text{if}\quad h>1. 
\end{cases}
\label{last_partial}
\end{align}
Then the partial derivative $\frac{\partial \frs({\mathbf x},\boldsymbol{\theta})}{\partial{\mathbf W}^{(h)}}$
can be computed as
$\frac{\partial \frs({\mathbf x},\boldsymbol{\theta})}{\partial{\mathbf W}^{(h)}}={\mathbf a}^T \frac{\partial{\mathbf x}^{(H)}}{\partial{\mathbf W}^{(h)}}$ for all layers $h=1,\dots,H$, where the partial $\frac{\partial{\mathbf x}^{(H)}}{\partial{\mathbf W}^{(h)}}$ is as given by \eqref{last_partial}. 
Consequently, we obtain

\begin{align}
{\mathbf G}_{ij}^{(h)}= 
\begin{cases}
{\mathbf a}^T
\mathlarger{\mathlarger{\mathlarger{\prod}}}_{l=2}^{H} \left[{\mathbf I}_m+\frac{\crs}{H\sqrt{m}}
\text{diag} [\phi'({\mathbf W}^{(l)} {\mathbf x}^{(l-1)}_i)]
{\mathbf W}^{(l)}
\right]\cdot \frac{\cph}{m}\,\,\text{diag} [\phi'({\mathbf  W}^{(1)} {\mathbf x}_i)]  \\ 
\quad\cdot\text{diag} [\phi'({\mathbf  W}^{(1)} {\mathbf x}_j)]\cdot
\mathlarger{\mathlarger{\mathlarger{\prod}}}_{l=2}^{H} \left[{\mathbf I}_m+\frac{\crs}{H\sqrt{m}}
{{\mathbf W}^{(l)}}^T 
\text{diag} [\phi'({\mathbf W}^{(l)} {\mathbf x}^{(l-1)}_j)]
\right]\,
{\mathbf a}\\ \quad
\cdot \left\langle
{\mathbf x}_i, {\mathbf x}_j\right\rangle,
\quad\text{if}\,\,h=1,\\
{\mathbf a}^T
\mathlarger{\mathlarger{\mathlarger{\prod}}}_{l=h+1}^{H} \left[{\mathbf I}_m+\frac{\crs}{H\sqrt{m}}
\text{diag} [\phi'({\mathbf W}^{(l)} {\mathbf x}_i^{(l-1)})]
{\mathbf W}^{(l)}
\right]\cdot\frac{\crs^2}{H^2 m}\,\,\text{diag} [\phi'({\mathbf  W}^{(h)} {\mathbf x}_i^{(h-1)})]  \\ \quad
\cdot \,\,\text{diag} [\phi'({\mathbf  W}^{(h)} {\mathbf x}_j^{(h-1)})]\,
\mathlarger{\mathlarger{\mathlarger{\prod}}}_{l=h+1}^{H} \left[{\mathbf I}_m+\frac{\crs}{H\sqrt{m}}
{{\mathbf W}^{(l)}}^T 
\text{diag} [\phi'({\mathbf W}^{(l)} {\mathbf x}_j^{(l-1)})]
\right]
{\mathbf a}\\
\quad\cdot \left\langle
{\mathbf x}^{(h-1)}_i, {\mathbf x}^{(h-1)}_j\right\rangle,\quad
\text{if}\,\,h>1.
\end{cases}
\label{entry_each_layer}
\end{align}

Under the bounded derivative assumption $|\phi'(z)|\leq B$ appearing in the statement of Lemma \ref{spectJ}, we conclude from \eqref{entry_each_layer} that
\begin{align}
\left|{\mathbf G}_{ij}^{(h)}\right|\leq
\begin{cases}
\|{\mathbf a}\|^2_{2}B^2\mathlarger{\frac{\cph}{m}}\,\,
\mathlarger{\mathlarger{\prod}}_{l=2}^H \left(1+\frac{\crs}{H\sqrt{m}}B
\|{\mathbf W}^{(l)}\|\right)^2 \left|\left\langle
{\mathbf x}_i, {\mathbf x}_j\right\rangle\right|, 
\quad &\text{if}\,\,h=1,\\
\|{\mathbf a}\|^2_{2}B^2\mathlarger{\frac{c^2_{res}}{H^2 m}}\,\,
\mathlarger{\mathlarger{\prod}}_{l=h+1}^H \left(1+\frac{\crs}{H\sqrt{m}}B
\|{\mathbf W}^{(l)}\|\right)^2 \left|\left\langle
{\mathbf x}^{(h-1)}_i, {\mathbf x}^{(h-1)}_j \right\rangle\right|, 
\quad &\text{if}\,\,h>1.
\end{cases}
\label{entry_ineq}
\end{align}

Now we are in a position to analyze $\alpha, \beta$ and
$L$ appearing in Theorem~\ref{GDthm} for the ResNet case.
Our first step would be to evaluate $\alpha$ for
random initialization of weight matrices,  
$\beta$ for a given $\boldsymbol\theta,$ and $L$
for a given pair of parameter sets $(\widetilde{\boldsymbol\theta}, {\boldsymbol\theta}).$
Then we will extend those results derived for single parameters to the ball ${\mathcal B}\left({\boldsymbol\theta}_0, \frac{4\twonorm{f({\boldsymbol\theta}_0)-{\mathbf y}}}{\alpha}\right)$ described in Theorem~\ref{GDthm}.

\subsection{Calculate lower bound}
\label{alpha_derivation}
We start our analysis with $\alpha$ and examine \eqref{posdef} for that purpose. We assume the weight matrices ${\mathbf W}^{(1)},\dots,{\mathbf W}^{(H)}$ are initialized with i.i.d. ${\mathcal N}(0,1)$ entries. Each matrix
${\mathbf G}^{(h)}$ in \eqref{posdef} is an inner product matrix, so
they have to be positive semidefinite. This would imply $\phi_{\min}({\mathcal J}(\boldsymbol{\theta}){\mathcal J}^T(\boldsymbol{\theta}))\geq 
\phi_{\min}({\mathbf G}^{(h)})$ for all $h=1,\dots, H.$ In particular, we
have $\phi_{\min}({\mathcal J}(\boldsymbol{\theta}){\mathcal J}^T(\boldsymbol{\theta}))\geq \phi_{\min}(
{\mathbf G}^{(1)}).$ Thus any lower bound for
$\phi_{\min}({\mathbf G}^{(1)})$ would be a lower bound for $\phi_{\min}({\mathcal J}(\boldsymbol{\theta}){\mathcal J}^T(\boldsymbol{\theta}))$ as well. The formula for the entries 
${\mathbf G}^{(1)}_{ij}$ of ${\mathbf G}^{(1)}$ is given by
\eqref{entry_each_layer}, and can be rewritten as

\begin{align}
{\mathbf G}_{ij}^{(1)}&=\frac{c_{\phi}}{m}\left\{
\left[
{\mathbf a}^T
\mathlarger{\mathlarger{\mathlarger{\prod}}}_{l=2}^{H} \left({\mathbf I}_m+\frac{\crs}{H\sqrt{m}}
\text{diag} [\phi'({\mathbf W}^{(l)} {\mathbf x}^{(l-1)}_i)]
{\mathbf W}^{(l)}
\right)\right] \odot [\phi'({\mathbf  W}^{(1)} {\mathbf x}_i)]^T
\right\} \nonumber\\  
&\quad\cdot
\left\{
[\phi'({\mathbf  W}^{(1)} {\mathbf x}_j)]\odot\left[
\mathlarger{\mathlarger{\mathlarger{\prod}}}_{l=2}^{H} \left({\mathbf I}_m+\frac{\crs}{H\sqrt{m}}
{{\mathbf W}^{(l)}}^T 
\text{diag} [\phi'({\mathbf W}^{(l)} {\mathbf x}^{(l-1)}_j)]\right)
\,{\mathbf a}\right] 
\right\} \left\langle
{\mathbf x}_i, {\mathbf x}_j\right\rangle
\label{entry_formula}
\end{align}

Denoting the data matrix having rows ${\mathbf x}_1,\dots,{\mathbf x}_n$ as ${\mathbf X}\in{\mathbb R}^{n\times d}$, and denoting
the matrix consisting of the rows ${\mathbf a}^T
\mathlarger{\mathlarger{\mathlarger{\prod}}}_{l=2}^{H} \left({\mathbf I}_m+\frac{\crs}{H\sqrt{m}}
\text{diag} [\phi'({\mathbf W}^{(l)} {\mathbf x}^{(l-1)}_i)]
{\mathbf W}^{(l)}
\right), i=1,\dots,n$ as ${\mathbf A}(\boldsymbol{\theta},{\mathbf X})\in{\mathbb R}^{n\times m}$,
it follows from \eqref{entry_formula} that 
\begin{align}
{\mathbf G}^{(1)}=\frac{c_{\phi}}{m}\left[
\left({\mathbf A}(\boldsymbol{\theta},{\mathbf X})\odot 
\phi'({\mathbf X} {{\mathbf W}^{(1)}}^T)\right)\cdot
\left({\mathbf A}(\boldsymbol{\theta},{\mathbf X})\odot 
\phi'({\mathbf X} {{\mathbf W}^{(1)}}^T)\right)^T
\right]\odot({\mathbf X}{\mathbf X}^T). 
\label{g1_compact}
\end{align}
Then let the rows of ${\mathbf W}^{(1)}$ be ${\mathbf w}_1,\dots,{\mathbf w}_{m}$ and the columns of ${\mathbf A}(\boldsymbol\theta,{\mathbf X})$ be ${\mathbf a}_1,\dots,{\mathbf a}_m$ to rewrite \eqref{g1_compact} as 
\begin{align}
{\mathbf G}^{(1)}=\frac{c_{\phi}}{m}\left(\sum_{l=1}^m [{\mathbf a}_l \odot \phi'({\mathbf X}{\mathbf w}_l)]
\cdot [{\mathbf a}_l \odot \phi'({\mathbf X}{\mathbf w}_l)]^T 
\right)\odot({\mathbf X}{\mathbf X}^T).
\label{g2_compact}
\end{align}
To analyze the Hadamard product given by \eqref{g2_compact}, we first focus on the term $[{\mathbf a}_l \odot \phi'({\mathbf X}{\mathbf w}_l)]\cdot [{\mathbf a}_l \odot \phi'({\mathbf X}{\mathbf w}_l)]^T$, and observe
\begin{gather}
[{\mathbf a}_l \odot \phi'({\mathbf X}{\mathbf w}_l)]\cdot [{\mathbf a}_l \odot \phi'({\mathbf X}{\mathbf w}_l)]^T
=({\mathbf a}_l {\mathbf a}^T_l)\odot
(\phi'({\mathbf X}{\mathbf w}_l) \phi'({\mathbf X}{\mathbf w}_l)^T)\nonumber\\
{\mathbf G}^{(1)}=\frac{c_{\phi}}{m} \mathlarger{\mathlarger{\sum}}_{l=1}^m ({\mathbf a}_l {\mathbf a}^T_l)\odot \left[(\phi'({\mathbf X}{\mathbf w}_l) \phi'({\mathbf X}{\mathbf w}_l)^T) \odot({\mathbf X}{\mathbf X}^T)\right]
\label{g1_expansion}
\end{gather}


We note that the randomness of the product 
$\left[(\phi'({\mathbf X}{\mathbf w}_l) \phi'({\mathbf X}{\mathbf w}_l)^T) \odot({\mathbf X}{\mathbf X}^T)\right]$ is due to the weight matrix ${\mathbf W}^{(1)}$ only, whereas the matrix ${\mathbf a}_l{\mathbf a}_l^T$ depends on the matrices ${\mathbf W}^{(2)},\dots,{\mathbf W}^{(H)}.$ Thus the expectation taken over the random initialization
of ${\mathbf W}^{(1)},\dots,{\mathbf W}^{(H)}$ 
can be expressed as
\begin{align}
{\mathbb E}\left\{
({\mathbf a}_l {\mathbf a}^T_l)\odot \left[(\phi'({\mathbf X}{\mathbf w}_l) \phi'({\mathbf X}{\mathbf w}_l)^T) \odot({\mathbf X}{\mathbf X}^T)\right] \right\} 
&= {\mathbb E}({\mathbf a}_l {\mathbf a}^T_l) \odot
{\mathbb E}\left[(\phi'({\mathbf X}{\mathbf w}_l) \phi'({\mathbf X}{\mathbf w}_l)^T) \odot({\mathbf X}{\mathbf X}^T)\right]
\nonumber\\
&= {\mathbb E}({\mathbf a}_l {\mathbf a}^T_l) \odot \Sigma({\mathbf X}) \label{sigmaX}
\end{align}
where we take $\Sigma({\mathbf X})\triangleq {\mathbb E}\left[(\phi'({\mathbf X}{\mathbf w}_l) \phi'({\mathbf X}{\mathbf w}_l)^T) \odot({\mathbf X}{\mathbf X}^T)\right]$ as defined in 
Lemma~\ref{minspectJ}. As a consequence of \eqref{sigmaX}, we have
\begin{align}
{\mathbb E}\left\{
\frac{1}{m}
\mathlarger{\mathlarger{\sum}}_{l=1}^m ({\mathbf a}_l {\mathbf a}^T_l)\odot \left[(\phi'({\mathbf X}{\mathbf w}_l) \phi'({\mathbf X}{\mathbf w}_l)^T) \odot({\mathbf X}{\mathbf X}^T)\right]
\right\}
=\left[\frac{1}{m} \mathlarger{\mathlarger{\sum}}_{l=1}^m {\mathbb E}({\mathbf a}_l {\mathbf a}^T_l)\right] 
\odot \Sigma({\mathbf X})
\label{sigmaX2}
\end{align}

Then we lower bound \cite{Schur1911} the minimum eigenvalue of \eqref{sigmaX2} as 
\begin{align}
\phi_{\min}\left[
{\mathbb E}\left\{
\frac{1}{m}
\mathlarger{\mathlarger{\sum}}_{l=1}^m ({\mathbf a}_l {\mathbf a}^T_l)\odot \left[(\phi'({\mathbf X}{\mathbf w}_l) \phi'({\mathbf X}{\mathbf w}_l)^T) \odot({\mathbf X}{\mathbf X}^T)\right]
\right\} \right]\geq \min_i
\frac{1}{m}\left[
\mathlarger{\mathlarger{\sum}}_{l=1}^m
{\mathbb E}|{\mathbf a}_l(i)|^2 \right] \,\, \lambda({\mathbf X}) 
\label{schur_app}
\end{align}
where $\lambda({\mathbf X})= \phi_{\min}(\Sigma({\mathbf X}))$
is as defined in Lemma~\ref{minspectJ}.

To estimate the term $\sum_{l=1}^m {\mathbb E}|{\mathbf a}_l(i)|^2$, we consider the rows of
${\mathbf A}(\boldsymbol\theta,{\mathbf X})$, the norms of which
can be bounded as
\begin{align}
&\left\|{\mathbf a}^T \mathlarger{\mathlarger{\mathlarger{\prod}}}_{j=2}^H 
\left({\mathbf I}_m+\frac{\crs}{H\sqrt{m}}
\text{diag} [\phi'({\mathbf W}^{(j)} {\mathbf x}^{(j-1)}_i)]
{\mathbf W}^{(j)} \right)\right\|_2^2\nonumber\\
\quad&
\geq
\|{\mathbf a}\|^2_2 \mathlarger{\mathlarger{\mathlarger{\prod}}}_{j=2}^H \phi^2_{\min}\left({\mathbf I}_m+\frac{\crs}{H\sqrt{m}}
\text{diag} [\phi'({\mathbf W}^{(j)} {\mathbf x}^{(j-1)}_i)]
{\mathbf W}^{(j)} \right)
\geq\|{\mathbf a}\|^2_2 \mathlarger{\mathlarger{\mathlarger{\prod}}}_{j=2}^H
\left(1-\frac{B \crs}{H \sqrt{m}}\|{\mathbf W}^{(j)}\|\right)^2
\remove{
&\leq \|{\mathbf a}\|^2_2 \mathlarger{\mathlarger{\mathlarger{\prod}}}_{j=2}^H \phi^2_{\max}\left({\mathbf I}_m+\frac{\crs}{H\sqrt{m}}
\text{diag} [\phi'({\mathbf W}^{(j)} {\mathbf x}^{(j-1)}_i)]
{\mathbf W}^{(j)} \right)
\leq
\|{\mathbf a}\|^2_2 \mathlarger{\mathlarger{\mathlarger{\prod}}}_{j=2}^H
\left(1+\frac{B \crs}{H \sqrt{m}}\|{\mathbf W}^{(j)}\|\right)^2}
\label{norm_ineq1}
\end{align}
Taking the expectation of both sides in \eqref{norm_ineq1}, we get
\begin{align}
\mathlarger{\mathlarger{\sum}}_{l=1}^m
{\mathbb E}|{\mathbf a}_l(i)|^2 
\geq \|{\mathbf a}\|^2_2 \mathlarger{\mathlarger{\mathlarger{\prod}}}_{j=2}^H
{\mathbb E}\left(1-\frac{B \crs}{H \sqrt{m}}\|{\mathbf W}^{(j)}\|\right)^2
\geq \|{\mathbf a}\|^2_2 \mathlarger{\mathlarger{\mathlarger{\prod}}}_{j=2}^H
\left(1-\frac{B \crs}{H \sqrt{m}}{\mathbb E}\|{\mathbf W}^{(j)}\|\right)^2
\label{norm_ineq2}
\end{align}

Since each entry of the weight matrix ${\mathbf W}^{(j)}$ is
sampled with the standard normal distribution, Gordon's theorem for Gaussian matrices gives us
$\frac{{\mathbb E}\|{\mathbf W}^{(j)}\|}{\sqrt{m}}\leq 2,$ for all $j=1,\dots,H.$ 
Therefore the inequality
\begin{align}
\|{\mathbf a}\|^2_2 \mathlarger{\mathlarger{\mathlarger{\prod}}}_{j=2}^H
\left(1-\frac{B \crs}{H \sqrt{m}}{\mathbb E}\|{\mathbf W}^{(j)}\|\right)^2
\geq \|{\mathbf a}\|^2_2
\mathlarger{\mathlarger{\mathlarger{\prod}}}_{j=2}^H
\left(1-\frac{2B \crs}{H} \right)^2
\geq \|{\mathbf a}\|^2_2 (1-\delta') {\text e}^{-4B\crs}
\remove{
{\text e}^{-4B(1+\delta)\crs}\leq
\mathlarger{\mathlarger{\mathlarger{\prod}}}_{j=2}^H
\left(1-\frac{2B(1+\delta) \crs}{H}\right)^2
\leq{\mathbb E}\mathlarger{\mathlarger{\mathlarger{\prod}}}_{j=2}^H
\left(1-\frac{B \crs}{H \sqrt{m}}\|{\mathbf W}^{(j)}\|\right)^2\nonumber\\
{\mathbb E}\mathlarger{\mathlarger{\mathlarger{\prod}}}_{j=2}^H
\left(1+\frac{B \crs}{H \sqrt{m}}\|{\mathbf W}^{(j)}\|\right)^2
\leq \mathlarger{\mathlarger{\mathlarger{\prod}}}_{j=2}^H
\left(1+\frac{2B(1+\delta)\crs}{H}\right)^2 \leq
{\text e}^{4B(1+\delta)\crs}
}
\label{norm_ineq3}
\end{align}
hold true for any given $\delta'$ between $0$ and $1$, and
$H$ being a function of $\delta'$ and large enough. 
Combining \eqref{norm_ineq2} with \eqref{norm_ineq3},
we observe that expected squared norm of each row of ${\mathbf A}(\boldsymbol\theta,{\mathbf X})$ is greater than 
$\|{\mathbf a}\|^2_2 (1-\delta') e^{-4B\crs}.$ \remove{and
$\|{\mathbf a}\|^2_2 e^{4B(1+\delta) \crs}.$} Hence we get the inequality
\begin{align}
\frac{1}{m}\sum_{l=1}^m {\mathbb E}|{\mathbf a}_l(i)|^2\geq 
(1-\delta')\frac{\|{\mathbf a}\|^2_2  
e^{-4B \crs}}{m}
\label{concentration}
\end{align}
for any given $i$, meaning that 
\begin{align}
\min_i \frac{1}{m}\left[\sum_{l=1}^m {\mathbb E}|{\mathbf a}_l(i)|^2\right] \geq (1-\delta')
\frac{\|{\mathbf a}\|^2_2  
e^{-4B \crs}}{m} 
\end{align}
is satisfied.
Then \eqref{schur_app} gives us
\begin{align}
\phi_{\min}\left[
{\mathbb E}\left\{
\frac{1}{m}
\mathlarger{\mathlarger{\sum}}_{l=1}^m ({\mathbf a}_l {\mathbf a}^T_l)\odot \left[(\phi'({\mathbf X}{\mathbf w}_l) \phi'({\mathbf X}{\mathbf w}_l)^T) \odot({\mathbf X}{\mathbf X}^T)\right]
\right\} \right]\geq (1-\delta')
\frac{\|{\mathbf a}\|^2_2  e^{-4B \crs}}{m}
\,\, \lambda({\mathbf X}).  \label{schur_app2} 
\end{align}
Applying matrix Chernoff bound for \eqref{g2_compact} yields the result 
\begin{align}
\phi_{\min}({\mathbf G}^{(1)}) \geq
(1-\delta')^2
c_{\phi}\frac{\|{\mathbf a}\|^2_2 e^{-4 B \crs}}{m} \lambda({\mathbf X})
\end{align}
with an exponentially high probability of the form 
$1-\kappa_1 m\,e^{-\kappa_2 m\delta'}$, for some constants $\kappa_1,\kappa_2.$
Then we lower bound the $\phi_{\min}[{\mathcal J}(\boldsymbol\theta_0)]$ corresponding to the random initialization $\boldsymbol\theta_0$ of the weight matrices as

\begin{align}
\phi_{\min}[{\mathcal J}(\boldsymbol\theta_0)]&=
\left(\phi_{\min}[{\mathcal J}(\boldsymbol\theta_0){\mathcal J}^T(\boldsymbol\theta_0)]\right)^{\frac{1}{2}}
\nonumber\\
&\geq 
\left(\phi_{\min}({\mathbf G}^{(1)})\right)^{\frac{1}{2}}
\nonumber\\
&\geq (1-\delta')\sqrt{\frac{c_{\phi}}{m}} 
\|{\mathbf a}\|_2 e^{-2B \crs} \sqrt{\lambda({\mathbf X})}.
\label{alpha_final_ineq}
\end{align}
with a high probability described above,
where $\delta'$ is any positive number less than $1,$ 
and $H$ depends on $\delta', B$ and $\crs$ only and 
sufficiently large.
So we can take the lower bound $\alpha_0$ for
the random initialization $\boldsymbol\theta_0$ to be

\begin{align}
\alpha_0=(1-\delta')\sqrt{\frac{c_{\phi}}{m}} 
\|{\mathbf a}\|_2 e^{-2B \crs} \sqrt{\lambda({\mathbf X})}.\label{alpha_final}    
\end{align}

The constant $\sqrt{\frac{c_{\phi}}{m}}$ appearing in \eqref{alpha_final} is due to the normalization applied
after the first layer, see \eqref{Res_Net_eq}. Moreover
comparing \eqref{alpha_final} with the $\alpha$ term stated in Lemma~\ref{minspectJ}, we observe there is an additional factor of $e^{-2B \crs}$ because of the ResNet structure.

\subsection{Calculate upper bound}

Our next step would be to focus on the $\beta$ parameter. We see from  \eqref{posdef} that $\|{\mathcal J}(\boldsymbol{\theta}){\mathcal J}^T(\boldsymbol{\theta}) \| \leq \sum_{h=1}^{H}\|{\mathbf G}^{(h)}\|$
holds, and derive
\begin{align}
\|{\mathcal J}(\boldsymbol{\theta}){\mathcal J}^T(\boldsymbol{\theta}) \| &\leq \sum_{h=1}^{H}\|{\mathbf G}^{(h)}\| 
\leq
\sum_{h=1}^{H}\|{\mathbf G}^{(h)}\|_F
=\sum_{h=1}^H\left(\sum_{i=1}^n\sum_{j=1}^n 
\left|{\mathbf G}_{ij}^{(h)}\right|^2\right)^{1/2}.
\label{sigma_max_bd}
\end{align}
Then inserting the inequality \eqref{entry_ineq} in \eqref{sigma_max_bd}, 
and letting the data matrix for the output of $k$-th layer be denoted as ${\mathbf X}^{(k)},$ we write
\begin{align}
\|{\mathcal J}(\boldsymbol{\theta}){\mathcal J}^T(\boldsymbol{\theta}) \|&\leq 
\|{\mathbf a}\|^2_{2}B^2\mathlarger{\frac{\cph}{m}}\,\,
\mathlarger{\mathlarger{\prod}}_{l=2}^H \left(1+\frac{\crs}{H\sqrt{m}}B
\|{\mathbf W}^{(l)}\|\right)^2\|{\mathbf X}{\mathbf X}^{T}\|_F
\nonumber\\
&\quad+\|{\mathbf a}\|^2_{2}B^2\mathlarger{\frac{\crs^2}{H^2 m}}\mathlarger{\mathlarger{\sum}}_{h=2}^H \,\,
\mathlarger{\mathlarger{\prod}}_{l=h+1}^H \left(1+\frac{\crs}{H\sqrt{m}}B
\|{\mathbf W}^{(l)}\|\right)^2 \|{\mathbf X}^{(h-1)}
{{\mathbf X}^{(h-1)}}^{T}\|_F\label{beta_for_resnet}
\end{align}

Consequently, $\beta$ for the ResNet model can be taken as the square root of the right hand side of \eqref{beta_for_resnet}. In order to
calculate $\beta$ more precisely, we need to express or upper bound the norm $\|{\mathbf X}^{(h-1)} {{\mathbf X}^{(h-1)}}^{T}\|_F$ appearing in \eqref{beta_for_resnet}. Therefore we write
\begin{align}
\left\|{\mathbf X}^{(h-1)} {{\mathbf X}^{(h-1)}}^{T}\right\|_F\leq 
\left\|{\mathbf X}^{(h-1)}\right\|_F^2=\sum_{i=1}^n \|{\mathbf x}_i^{(h-1)}\|_2^2
\label{frob_upper}
\end{align}
Then each squared norm of the $(h-1)$-th layer output contributing to
the summation in \eqref{frob_upper} can be upper bounded as
\begin{align}
\|{\mathbf x}_i^{(h-1)}\|_2 &\leq \sqrt{\frac{c_{\phi}}{m}}B
\|{\mathbf W}^{(1)}\|\|{\mathbf x_i}\|_2\left(1+ 
\mathlarger{\mathlarger{\sum}}_{\substack{\{h_1,\dots,h_k\}\\
\subseteq\{2,\dots,h-1\}}} \left(\frac{\crs}{H\sqrt{m}}\right)^k
B^k \|{\mathbf W}^{(h_k)}\|\dots \|{\mathbf W}^{(h_1)}\|
\right)
\nonumber\\
&=\sqrt{\frac{c_{\phi}}{m}}B
\|{\mathbf W}^{(1)}\|\|{\mathbf x_i}\|_2
\mathlarger{\mathlarger{\prod}}_{j=2}^{h-1}\left(1+
\frac{\crs}{H\sqrt{m}}B \|{\mathbf W}^{(j)}\|\right)
\label{frob_upper2}
\end{align}
using \eqref{ResNet_uppbd}. Combining \eqref{frob_upper} with
\eqref{frob_upper2}, we get
\begin{align}
\left\|{\mathbf X}^{(h-1)} {{\mathbf X}^{(h-1)}}^{T}\right\|_F
\leq \frac{c_{\phi}}{m} B^2   
\|{\mathbf W}^{(1)}\|^2 \mathlarger{\mathlarger{\prod}}_{j=2}^{h-1}\left(1+
\frac{\crs}{H\sqrt{m}}B \|{\mathbf W}^{(j)}\|\right)^2
\|{\mathbf X}\|_F^2. \label{frob_upper3}
\end{align}
Then we insert \eqref{frob_upper3} in \eqref{beta_for_resnet} to conclude
\begin{align}
\|{\mathcal J}(\boldsymbol\theta){\mathcal J}^T(\boldsymbol\theta)\|
\leq \left(\|{\mathbf a}\|^2_{2}B^2\mathlarger{\frac{\cph}{m}}\,\,
+\|{\mathbf a}\|^2_{2}B^4
\|{\mathbf W}^{(1)}\|^2\mathlarger{\frac{c_{\phi}\crs^2}{H m^2}}
\right)\mathlarger{\mathlarger{\prod}}_{l=2}^H \left(1+\frac{\crs}{H\sqrt{m}}B
\|{\mathbf W}^{(l)}\|\right)^2
\|{\mathbf X}\|_F^2.\label{beta_inter}
\end{align}
Letting the set of inequalities $\|{\mathbf W}^{(j)}\|\leq A, j=1,\dots,H$
be satisfied for some $A$, it follows from
\eqref{beta_inter} that
\begin{align}
\|{\mathcal J}(\boldsymbol\theta){\mathcal J}^T(\boldsymbol\theta)\|
\leq \left(\|{\mathbf a}\|^2_{2}B^2\mathlarger{\frac{\cph}{m}}\,\,
+\|{\mathbf a}\|^2_{2}B^4
A^2\mathlarger{\frac{c_{\phi}\crs^2}{H m^2}}\right) 
e^{\frac{2AB\crs}{\sqrt{m}}} \|{\mathbf X}\|_F^2  
\end{align}
Hence using the inequality $\sqrt{a^2+b^2}\leq a+b$ for positive $a$ and $b$, $\beta$ for the ResNet can be taken as
\begin{align}
\beta=\|{\mathbf a}\|_{2}\left(B\sqrt{\frac{\cph}{m}}\,\,
+A\,B^2 
\frac{\sqrt{c_{\phi}}\crs}{\sqrt{H} m}\right) 
e^{\frac{AB\crs}{\sqrt{m}}}
\|{\mathbf X}\|_F\label{beta_final}
\end{align}

We see from \eqref{beta_final} that $\beta$ satisfies $\beta=O(\|{\mathbf a}\|_{2}\exp(B)\|{\mathbf X}\|),$ as opposed to the result for the single hidden layer model $\beta=O(\|{\mathbf a}\|_{2}B\|{\mathbf X}\|)$ given by Lemma~\ref{spectJ}. Note also that $\beta$ does not increase with the number of layers $H$, meaning that larger depths for the ResNet does not have a negative effect on the norm of $\|{\mathcal J}(\boldsymbol{\theta})\|.$

\subsection{Calculate Lipschitzness}

We continue our analysis with finding a Lipschitz constant $L$ for the Jacobian ${\mathcal J}(\boldsymbol{\theta})$ of the ResNet. In order to derive such a constant $L$, we need to analyze the norm of the difference 
${\mathcal J}(\widetilde{\boldsymbol\theta})-{\mathcal J} (\boldsymbol\theta)$ given by the formula
\begin{align}
{\mathcal J}(\widetilde{\boldsymbol\theta})-{\mathcal J} (\boldsymbol\theta)=
\begin{bmatrix}
{\mathbf a}^T\left(\frac{\partial{\mathbf x}_1^{(H)}}{\partial{\mathbf W}^{(1)}}[\widetilde{\boldsymbol\theta}]-\frac{\partial{\mathbf x}_1^{(H)}}{\partial{\mathbf W}^{(1)}}[{\boldsymbol\theta}]\right) & \dots & {\mathbf a}^T\left(\frac{\partial{\mathbf x}_1^{(H)}}{\partial{\mathbf W}^{(H)}}[\widetilde{\boldsymbol\theta}]-\frac{\partial{\mathbf x}_1^{(H)}}{\partial {\mathbf W}^{(H)}}[\boldsymbol\theta]\right)\\
\vdots & \dots & \vdots \\
{\mathbf a}^T\left(\frac{\partial{\mathbf x}_n^{(H)}}{\partial{\mathbf W}^{(1)}}[\widetilde{\boldsymbol\theta}]-\frac{\partial{\mathbf x}_n^{(H)}}{\partial {\mathbf W}^{(1)}}[\boldsymbol\theta]\right) & \dots & {\mathbf a}^T\left(\frac{\partial{\mathbf x}_n^{(H)}}{\partial{\mathbf W}^{(H)}}[\widetilde{\boldsymbol\theta}]
-\frac{\partial{\mathbf x}_n^{(H)}}{\partial {\mathbf W}^{(H)}}[{\boldsymbol\theta}]\right)
\end{bmatrix}
\label{jacob_diff}
\end{align}
where $\frac{\partial{\mathbf x}_i^{(H)}}{\partial{\mathbf W}^{(1)}}[\widetilde{\boldsymbol\theta}]-\frac{\partial{\mathbf x}_i^{(H)}}{\partial{\mathbf W}^{(1)}}[{\boldsymbol\theta}]\in {\mathbb R}^{m\times md}$ and $\frac{\partial{\mathbf x}_i^{(H)}}{\partial{\mathbf W}^{(h)}}[\widetilde{\boldsymbol\theta}]-\frac{\partial{\mathbf x}_i^{(H)}}{\partial{\mathbf W}^{(h)}}[{\boldsymbol\theta}]\in {\mathbb R}^{m\times m^2}$, for all $i=1,\dots,n$ and $h=2,\dots,H$, as stated in the beginning of the section. The norm of \eqref{jacob_diff} can be upper bounded as
\begin{align}
\|{\mathcal J}(\widetilde{\boldsymbol\theta})-{\mathcal J} (\boldsymbol\theta)\|^2 &\leq 
\|{\mathcal J}(\widetilde{\boldsymbol\theta})-{\mathcal J} (\boldsymbol\theta)\|^2_F \nonumber\\&= \mathlarger{\mathlarger{\sum}}_{i=1}^n \mathlarger{\mathlarger{\sum}}_{h=1}^H 
\left\|{\mathbf a}^T\left(\frac{\partial{\mathbf x}_i^{(H)}}{\partial{\mathbf W}^{(h)}}[\widetilde{\boldsymbol\theta}]-\frac{\partial{\mathbf x}_i^{(H)}}{\partial{\mathbf W}^{(h)}}[\boldsymbol\theta]\right)\right\|_2^2 \nonumber\\
\|{\mathcal J}(\widetilde{\boldsymbol\theta})-{\mathcal J} (\boldsymbol\theta)\| &\leq 
\mathlarger{\mathlarger{\sum}}_{h=1}^H
\sqrt{
\mathlarger{\mathlarger{\sum}}_{i=1}^n 
\left\|{\mathbf a}^T\left(\frac{\partial{\mathbf x}_i^{(H)}}{\partial{\mathbf W}^{(h)}}[\widetilde{\boldsymbol\theta}]-\frac{\partial{\mathbf x}_i^{(H)}}{\partial{\mathbf W}^{(h)}}[\boldsymbol\theta]\right)\right\|^2_2
}
\label{L_ineq_first}
\end{align}
where we use the basic inequality $\sqrt{b_1^2+\dots+b_H^2}\leq
|b_1|+\dots+|b_H|$ in \eqref{L_ineq_first}.

The next step is to consider the difference of partials $\frac{\partial{\mathbf x}_i^{(H)}}{\partial{\mathbf W}^{(h)}}[\widetilde{\boldsymbol\theta}]-\frac{\partial{\mathbf x}_i^{(H)}}{\partial{\mathbf W}^{(h)}}[\boldsymbol\theta].$  
We start with the case $h=1.$
Recall from \eqref{last_partial} that
\begin{align*}
\frac{\partial{\mathbf x}_i^{(H)}}{\partial{\mathbf W}^{(1)}}
[\boldsymbol\theta]
=
\mathlarger{\mathlarger{\mathlarger{\prod}}}_{l=2}^{H} \left[{\mathbf I}_m+\frac{\crs}{H\sqrt{m}}
\text{diag} [\phi'({\mathbf W}^{(l)} {\mathbf x}_i^{(l-1)}(\boldsymbol\theta))]
{\mathbf W}^{(l)}
\right]\cdot \sqrt{\frac{\cph}{m}}\,\,
\text{diag} [\phi'({\mathbf W}^{(1)} {\mathbf x}_i)]
\begin{bmatrix}
{\mathbf x}_i^T & \boldsymbol{0} & \boldsymbol{0}\\
\vdots &\ddots & \vdots\\
\boldsymbol{0} & \boldsymbol{0} &{\mathbf x}_i^T \\
\end{bmatrix}.    
\end{align*}
Then we can expand the expression  $\frac{\partial{\mathbf x}_i^{(H)}}{\partial{\mathbf W}^{(1)}}[\widetilde{\boldsymbol\theta}]-\frac{\partial{\mathbf x}_i^{(H)}}{\partial{\mathbf W}^{(1)}}[\boldsymbol\theta]$ as
\begin{align}
&\frac{\partial{\mathbf x}_i^{(H)}}{\partial{\mathbf W}^{(1)}}[\widetilde{\boldsymbol\theta}]-\frac{\partial{\mathbf x}_i^{(H)}}{\partial{\mathbf W}^{(1)}}[\boldsymbol\theta]
=
\mathlarger{\mathlarger{\mathlarger{\prod}}}_{l=2}^{H} \left[{\mathbf I}_m+\frac{\crs}{H\sqrt{m}}
\text{diag} [\phi'(\widetilde{{\mathbf W}}^{(l)} {\mathbf x}_i^{(l-1)}(\widetilde{\boldsymbol\theta}))]
\widetilde{{\mathbf W}}^{(l)}
\right]\nonumber\\&\,\,\cdot \sqrt{\frac{\cph}{m}}\,\,
\left(\text{diag} [\phi'(\widetilde{{\mathbf W}}^{(1)} {\mathbf x}_i)
]
-\text{diag} 
[\phi'({\mathbf W}^{(1)} {\mathbf x}_i)]\right)
\begin{bmatrix}
{\mathbf x}_i^T & \boldsymbol{0} & \boldsymbol{0}\\
\vdots &\ddots & \vdots\\
\boldsymbol{0} & \boldsymbol{0} &{\mathbf x}_i^T \\
\end{bmatrix}
\nonumber\\&\,\,
+\Bigg\{\mathlarger{\mathlarger{\mathlarger{\prod}}}_{l=2}^{H} \left[{\mathbf I}_m+\frac{\crs}{H\sqrt{m}}
\text{diag} [\phi'(\widetilde{{\mathbf W}}^{(l)} {\mathbf x}_i^{(l-1)}(\widetilde{\boldsymbol\theta}))]
\widetilde{{\mathbf W}}^{(l)}
\right]\nonumber\\&\quad- \mathlarger{\mathlarger{\mathlarger{\prod}}}_{l=2}^{H} \left[{\mathbf I}_m+\frac{\crs}{H\sqrt{m}}
\text{diag} [\phi'({\mathbf W}^{(l)} {\mathbf x}_i^{(l-1)}(\boldsymbol\theta))]
{\mathbf W}^{(l)}\right]\Bigg\}
\cdot \sqrt{\frac{\cph}{m}}\text{diag} 
[\phi'({\mathbf W}^{(1)} {\mathbf x}_i)]
\begin{bmatrix}
{\mathbf x}_i^T & \boldsymbol{0} & \boldsymbol{0}\\
\vdots &\ddots & \vdots\\
\boldsymbol{0} & \boldsymbol{0} &{\mathbf x}_i^T \\
\end{bmatrix}.
\label{diff_partial}
\end{align}

Equation \eqref{diff_partial} implies
the matrix product ${\mathbf a}^T \left(\frac{\partial{\mathbf x}_i^{(H)}}{\partial{\mathbf W}^{(1)}}[\widetilde{\boldsymbol\theta}]-\frac{\partial{\mathbf x}_i^{(H)}}{\partial{\mathbf W}^{(1)}}[\boldsymbol\theta]\right)$ 
has the form
\begin{align}
&{\mathbf a}^T \left(\frac{\partial{\mathbf x}_i^{(H)}}{\partial{\mathbf W}^{(1)}}[\widetilde{\boldsymbol\theta}]-\frac{\partial{\mathbf x}_i^{(H)}}{\partial{\mathbf W}^{(1)}}[\boldsymbol\theta]\right)
\nonumber\\
&\quad \triangleq {\mathbf a}_{(1,H)}^T({\mathbf x}_i,\widetilde{\boldsymbol\theta})
\sqrt{\frac{\cph}{m}}\,\,
\left(\text{diag} [\phi'(\widetilde{{\mathbf W}}^{(1)} {\mathbf x}_i)
]
-\text{diag} 
[\phi'({\mathbf W}^{(1)} {\mathbf x}_i)]\right)
\begin{bmatrix}
{\mathbf x}_i^T & \boldsymbol{0} & \boldsymbol{0}\\
\vdots &\ddots & \vdots\\
\boldsymbol{0} & \boldsymbol{0} &{\mathbf x}_i^T \\
\end{bmatrix}
\nonumber\\&
\quad\quad+\left[
{\mathbf a}_{(1,H)}^T({\mathbf x}_i,\widetilde{\boldsymbol\theta})-{\mathbf a}_{(1,H)}^T({\mathbf x}_i,\boldsymbol\theta)\right] \sqrt{\frac{\cph}{m}}\text{diag} 
[\phi'({\mathbf W}^{(1)} {\mathbf x}_i)]
\begin{bmatrix}
{\mathbf x}_i^T & \boldsymbol{0} & \boldsymbol{0}\\
\vdots &\ddots & \vdots\\
\boldsymbol{0} & \boldsymbol{0} &{\mathbf x}_i^T \\
\end{bmatrix}
\nonumber\\
&\quad\triangleq{\mathbf v}_{(1,H)}^{(1)}({\mathbf x}_i)+{\mathbf v}_{(1,H)}^{(2)}({\mathbf x}_i)
\label{H1_decomp}
\end{align}
where we take the row vectors 
${\mathbf a}_{(h_i,h_f)}^T({\mathbf x}_i,\boldsymbol{\theta}),{\mathbf v}_{(1,H)}^{(1)}({\mathbf x}_i),{\mathbf v}_{(1,H)}^{(2)}({\mathbf x}_i)$ to be

\begin{gather}
{\mathbf a}_{(h_i,h_f)}^T({\mathbf x}_i,\boldsymbol{\theta})\triangleq
{\mathbf a}^T
\mathlarger{\mathlarger{\mathlarger{\prod}}}_{l=h_i+1}^{h_f} \left[{\mathbf I}_m+\frac{\crs}{H\sqrt{m}}
\text{diag} [\phi'({\mathbf W}^{(l)} {\mathbf x}_i^{(l-1)}(\boldsymbol\theta))]
{\mathbf W}^{(l)}\right],
\end{gather}
\begin{gather}
{\mathbf v}_{(1,H)}^{(1)}({\mathbf x}_i)\triangleq
{\mathbf a}_{(1,H)}^T({\mathbf x}_i,\widetilde{\boldsymbol\theta})
\sqrt{\frac{\cph}{m}}\,\,
\left(\text{diag} [\phi'(\widetilde{{\mathbf W}}^{(1)} {\mathbf x}_i)
]
-\text{diag} 
[\phi'({\mathbf W}^{(1)} {\mathbf x}_i)]\right)
\begin{bmatrix}
{\mathbf x}_i^T & \boldsymbol{0} & \boldsymbol{0}\\
\vdots &\ddots & \vdots\\
\boldsymbol{0} & \boldsymbol{0} &{\mathbf x}_i^T \\
\end{bmatrix}, 
\end{gather}
\begin{gather}
{\mathbf v}_{(1,H)}^{(2)}({\mathbf x}_i)\triangleq
\left[
{\mathbf a}_{(1,H)}^T({\mathbf x}_i,\widetilde{\boldsymbol\theta})-{\mathbf a}_{(1,H)}^T({\mathbf x}_i,\boldsymbol\theta)\right] \sqrt{\frac{\cph}{m}}\text{diag}
[\phi'({\mathbf W}^{(1)} {\mathbf x}_i)]
\begin{bmatrix}
{\mathbf x}_i^T & \boldsymbol{0} & \boldsymbol{0}\\
\vdots &\ddots & \vdots\\
\boldsymbol{0} & \boldsymbol{0} &{\mathbf x}_i^T \\
\end{bmatrix}.
\label{a_1H_defn}
\end{gather}

It follows from Lemma~6.8 of \cite{oymak2019towards} that
\begin{align}
\left\|{\mathbf v}_{(1,H)}^{(1)}({\mathbf x}_i)\right\|_2\leq 
\sqrt{\frac{\cph}{m}}
\left\|{\mathbf a}_{(1,H)}({\mathbf x}_i,\widetilde{\boldsymbol\theta})\right\|_{\infty} M\|{\mathbf x}_i\|_2\left\| \widetilde{{\mathbf W}}^{(1)}-{\mathbf W}^{(1)}\right\|_F, 
\label{lemma6_8}
\end{align}
and the norm of the vector ${\mathbf v}_{(1,H)}^{(2)}({\mathbf x}_i)$ can straightforwardly be bounded as
\begin{align}
\left\|{\mathbf v}_{(1,H)}^{(2)}({\mathbf x}_i)\right\|_2\leq \sqrt{\frac{\cph}{m}} 
\left\|{\mathbf a}_{(1,H)}({\mathbf x}_i,\widetilde{\boldsymbol\theta})-{\mathbf a}_{(1,H)}({\mathbf x}_i,\boldsymbol\theta)\right\|_{\infty}B\|{\mathbf x}_i\|_2.
\label{diff_infty}
\end{align}
To provide a bound for the norm term $\left\|{\mathbf a}_{(1,H)}({\mathbf x}_i,\widetilde{\boldsymbol\theta})\right\|_{\infty}$ 
in \eqref{lemma6_8}, we write

\begin{align}
\left\|{\mathbf a}_{(1,H)}({\mathbf x}_i,\widetilde{\boldsymbol\theta})\right\|_{\infty}&\leq \|{\mathbf a}\|_{\infty}\left\|\mathlarger{\mathlarger{\mathlarger{\prod}}}_{l=2}^{H} 
\left[{\mathbf I}_m+\frac{\crs}{H\sqrt{m}}
\text{diag} [\phi'(\widetilde{{\mathbf W}}^{(l)} {\mathbf x}_i^{(l-1)}(\widetilde{\boldsymbol\theta}))]
\widetilde{{\mathbf W}}^{(l)}\right] \right\|_1
\nonumber\\
&\leq \|{\mathbf a}\|_{\infty} \sqrt{m} \left\|\mathlarger{\mathlarger{\mathlarger{\prod}}}_{l=2}^{H} 
\left[{\mathbf I}_m+\frac{\crs}{H\sqrt{m}}
\text{diag} [\phi'(\widetilde{{\mathbf W}}^{(l)} {\mathbf x}_i^{(l-1)}(\widetilde{\boldsymbol\theta}))
\widetilde{{\mathbf W}}^{(l)}\right] \right\|_2\nonumber
\\
\left\|{\mathbf a}_{(1,H)}({\mathbf x}_i,\widetilde{\boldsymbol\theta})\right\|_{\infty}& \leq \sqrt{m}\|{\mathbf a}\|_{\infty}
\mathlarger{\mathlarger{\mathlarger{\prod}}}_{l=2}^{H}
\left\|{\mathbf I}_m+\frac{\crs}{H\sqrt{m}}
\text{diag} [\phi'(\widetilde{{\mathbf W}}^{(l)} {\mathbf x}_i^{(l-1)}(\widetilde{\boldsymbol\theta}))
\widetilde{{\mathbf W}}^{(l)}\right\|_2
\nonumber\\
&\leq \sqrt{m}\|{\mathbf a}\|_{\infty}
\mathlarger{\mathlarger{\mathlarger{\prod}}}_{l=2}^{H}
\left(1+\frac{B\crs}{H\sqrt{m}}
\|\widetilde{{\mathbf W}}^{(l)}\| \right).
\label{single_infty_bd}
\end{align}
Inserting \eqref{single_infty_bd} in \eqref{lemma6_8}, we get
\begin{align}
\left\|{\mathbf v}_{(1,H)}^{(1)}({\mathbf x}_i)\right\|_2\leq 
\sqrt{\cph}\|{\mathbf a}\|_{\infty}
\mathlarger{\mathlarger{\mathlarger{\prod}}}_{l=2}^{H}
\left(1+\frac{B\crs}{H\sqrt{m}}
\|\widetilde{{\mathbf W}}^{(l)}\| \right)
M\|{\mathbf x}_i\|_2 \left\|\widetilde{{\mathbf W}}^{(1)}-{\mathbf W}^{(1)}\right\|_F.\label{v1H1_bd}
\end{align}

We proceed our analysis with the term $\left\|{\mathbf v}_{(1,H)}^{(2)}({\mathbf x}_i)\right\|.$
To analyze the norm $$\left\|{\mathbf a}_{(1,H)}({\mathbf x}_i,\widetilde{\boldsymbol\theta})-{\mathbf a}_{(1,H)}({\mathbf x}_i,\boldsymbol\theta)\right\|_{\infty}$$ appearing in \eqref{diff_infty}, we define the matrix
\begin{align}
{\mathbf D}_{(h_i,h_f)}({\mathbf x}_i)&\triangleq \mathlarger{\mathlarger{\mathlarger{\prod}}}_{l=h_i+1}^{h_f} 
\left[{\mathbf I}_m+\frac{\crs}{H\sqrt{m}}
\text{diag} [\phi'(\widetilde{{\mathbf W}}^{(l)} {\mathbf x}_i^{(l-1)}(\widetilde{\boldsymbol\theta}))]
\widetilde{{\mathbf W}}^{(l)}\right]\nonumber\\&\,\,- \mathlarger{\mathlarger{\mathlarger{\prod}}}_{l=h_i+1}^{h_f} \left[{\mathbf I}_m+\frac{\crs}{H\sqrt{m}}
\text{diag} [\phi'({\mathbf W}^{(l)} {\mathbf x}_i^{(l-1)}(\boldsymbol\theta))]
{\mathbf W}^{(l)}\right]  
\label{D_matrix}
\end{align}
and observe
\begin{align}
{\mathbf a}_{(1,H)}({\mathbf x}_i,\widetilde{\boldsymbol\theta})
-{\mathbf a}_{(1,H)}({\mathbf x}_i,\boldsymbol\theta)&=
{\mathbf a}^T {\mathbf D}_{(1,H)}({\mathbf x}_i)\nonumber\\
\left\|{\mathbf a}_{(1,H)}({\mathbf x}_i,\widetilde{\boldsymbol\theta})-
{\mathbf a}_{(1,H)}
({\mathbf x}_i,\boldsymbol\theta)\right\|_{\infty}&\leq\|{\mathbf a}\|_{\infty}\|{\mathbf D}_{(1,H)}({\mathbf x}_i)\|_1
\leq \sqrt{m}\|{\mathbf a}\|_{\infty}\|{\mathbf D}_{(1,H)}({\mathbf x}_i)\|_2 \nonumber\\
\left\|{\mathbf v}_{(1,H)}^{(2)}({\mathbf x}_i)\right\|
&\leq \sqrt{\cph}\|{\mathbf a}\|_{\infty}\|{\mathbf D}_{(1,H)}({\mathbf x}_i)\|_2 B \|{\mathbf x}_i\|_2.\label{v1H2_bd}
\end{align}
Our next goal would be to derive an upper bound for $\|{\mathbf D}_{(1,H)}({\mathbf x}_i)\|_2.$ For that purpose, we write
\begin{align}
{\mathbf D}_{(1,H)}({\mathbf x}_i)&= \left[{\mathbf I}_m+\frac{\crs}{H\sqrt{m}}
\text{diag} [\phi'(\widetilde{{\mathbf W}}^{(H)} {\mathbf x}_i^{(H-1)}(\widetilde{\boldsymbol\theta}))]
\widetilde{{\mathbf W}}^{(H)}\right]{\mathbf D}_{(1,H-1)}({\mathbf x}_i)
\nonumber\\&\quad
+\frac{\crs}{H\sqrt{m}}\left(\text{diag} [\phi'(\widetilde{{\mathbf W}}^{(H)} {\mathbf x}_i^{(H-1)}(\widetilde{\boldsymbol\theta}))]
\widetilde{{\mathbf W}}^{(H)}-\text{diag} [\phi'({\mathbf W}^{(H)} {\mathbf x}_i^{(H-1)}(\boldsymbol\theta))]{\mathbf W}^{(H)}\right)
\nonumber\\&\quad\cdot
\mathlarger{\mathlarger{\mathlarger{\prod}}}_{l=2}^{H-1} 
\left[{\mathbf I}_m+\frac{\crs}{H\sqrt{m}}
\text{diag} [\phi'({\mathbf W}^{(l)} {\mathbf x}_i^{(l-1)}(\boldsymbol\theta))]
{\mathbf W}^{(l)}\right]\nonumber 
\end{align}
which implies

\begin{align}
\|{\mathbf D}_{(1,H)}({\mathbf x}_i)\|_2&\leq   
\left(1+\frac{B\crs}{H\sqrt{m}}
\|\widetilde{{\mathbf W}}^{(H)}\| \right)
\|{\mathbf D}_{(1,H-1)}({\mathbf x}_i)\|_2
\nonumber\\&\quad
+\frac{\crs}{H\sqrt{m}}\left\|\text{diag} [\phi'(\widetilde{{\mathbf W}}^{(H)} {\mathbf x}_i^{(H-1)}(\widetilde{\boldsymbol\theta}))]
\widetilde{{\mathbf W}}^{(H)}-\text{diag} [\phi'({\mathbf W}^{(H)} {\mathbf x}_i^{(H-1)}(\boldsymbol\theta))]{\mathbf W}^{(H)}\right\|
\nonumber\\&\quad\quad
\cdot
\mathlarger{\mathlarger{\mathlarger{\prod}}}_{l=2}^{H-1} 
\left(1+\frac{B\crs}{H\sqrt{m}}
\|\widetilde{{\mathbf W}}^{(l)}\|\right)\label{norm_D_1H}
\end{align}
We see from \eqref{norm_D_1H} that it is possible to derive a recursive upper bound for $\|{\mathbf D}_{(1,H)}({\mathbf x}_i)\|_2$ provided that we simplify or bound the norm given by 
$$
\left\|
\text{diag} [\phi'(\widetilde{{\mathbf W}}^{(l)} {\mathbf x}_i^{(l-1)}(\widetilde{\boldsymbol\theta}))]
\widetilde{{\mathbf W}}^{(l)}-\text{diag} [\phi'({\mathbf W}^{(l)} {\mathbf x}_i^{(l-1)}(\boldsymbol\theta))]{\mathbf W}^{(l)}\right\|,   \,l=H 
$$
in terms of the network input ${\mathbf x}_i$, and the set of ResNet parameters $\boldsymbol\theta$ and
$\widetilde{\boldsymbol\theta}.$
So using triangle inequality, we write

\begin{align}
&\left\|
\text{diag} [\phi'(\widetilde{{\mathbf W}}^{(l)} {\mathbf x}_i^{(l-1)}(\widetilde{\boldsymbol\theta}))]
\widetilde{{\mathbf W}}^{(l)}-\text{diag} [\phi'({\mathbf W}^{(l)} {\mathbf x}_i^{(l-1)}(\boldsymbol\theta))]{\mathbf W}^{(l)}\right\|\nonumber\\ & \quad
\leq
\left\|\text{diag} [\phi'(\widetilde{{\mathbf W}}^{(l)} {\mathbf x}_i^{(l-1)}(\widetilde{\boldsymbol\theta}))
]\right\|\cdot
\|\widetilde{{\mathbf W}}^{(l)}-
{\mathbf W}^{(l)}\|\nonumber\\& \quad\quad+
\left\|\text{diag} [\phi'(\widetilde{{\mathbf W}}^{(l)} {\mathbf x}_i^{(l-1)}(\widetilde{\boldsymbol\theta}))]
-\text{diag} [\phi'(\widetilde{{\mathbf W}}^{(l)} {\mathbf x}_i^{(l-1)}(\boldsymbol\theta))]
\right\|\cdot \|{\mathbf W}^{(l)}\|
\nonumber\\& \quad\quad+
\left\|\text{diag} [\phi'(\widetilde{{\mathbf W}}^{(l)} {\mathbf x}_i^{(l-1)}(\boldsymbol\theta))]
-\text{diag} [\phi'({\mathbf W}^{(l)} {\mathbf x}_i^{(l-1)}(\boldsymbol\theta))]
\right\|\cdot \|{\mathbf W}^{(l)}\|
\label{diag_diff}
\end{align}
To analyze the difference of diagonal matrices $\text{diag} [\phi'(\widetilde{{\mathbf W}}^{(l)} {\mathbf x}_i^{(l-1)}(\widetilde{\boldsymbol\theta}))]
-\text{diag} [\phi'(\widetilde{{\mathbf W}}^{(l)} {\mathbf x}_i^{(l-1)}(\boldsymbol\theta))]$ appearing in \eqref{diag_diff}, 
we use the mean value theorem for each diagonal entry and observe
the equation
\begin{align}
&\left|\phi'\left(\widetilde{{\mathbf W}}^{(l)} {\mathbf x}_i^{(l-1)}(\widetilde{\boldsymbol\theta})\right)_j-
\phi'\left(\widetilde{{\mathbf W}}^{(l)} {\mathbf x}_i^{(l-1)}(\boldsymbol\theta)\right)_j\right|\nonumber\\
&\quad=\left|\phi''\left(\widetilde{\mathbf w}^{(l)}_j\left[\alpha_{i,j} {\mathbf x}_i^{(l-1)}(\widetilde{\boldsymbol\theta})+
(1-\alpha_{i,j}){\mathbf x}_i^{(l-1)}(\boldsymbol\theta)\right]\right)\right|\,\,
\cdot
\left|\widetilde{\mathbf w}^{(l)}_j\left({\mathbf x}_i^{(l-1)}(\widetilde{\boldsymbol\theta})-{\mathbf x}_i^{(l-1)}(\boldsymbol\theta)\right)\right|\label{diag_diff2}
\end{align}
is valid for some $\alpha_{i,j}\in[0,1],$ where
$\widetilde{\mathbf w}^{(l)}_j$ refers to the $j$-th
row of $\widetilde{{\mathbf W}}^{(l)}.$
Under the assumption of
$|\phi''(z)|\leq M$ stated in Lemma~\ref{JLlem}, it follows from
\eqref{diag_diff2} that
\begin{gather}
\left|\phi'\left(\widetilde{{\mathbf W}}^{(l)} {\mathbf x}_i^{(l-1)}(\widetilde{\boldsymbol\theta})\right)_j-
\phi'\left(\widetilde{{\mathbf W}}^{(l)} {\mathbf x}_i^{(l-1)}(\boldsymbol\theta)\right)_j\right|
\leq M \left\|\widetilde{\mathbf w}^{(l)}_j\right\|_2 \left\|{\mathbf x}_i^{(l-1)}(\widetilde{\boldsymbol\theta})-{\mathbf x}_i^{(l-1)}(\boldsymbol\theta)\right\|_2\nonumber\\
\left\|\text{diag} [\phi'(\widetilde{{\mathbf W}}^{(l)} {\mathbf x}_i^{(l-1)}(\widetilde{\boldsymbol\theta}))]
-\text{diag} [\phi'(\widetilde{{\mathbf W}}^{(l)} {\mathbf x}_i^{(l-1)}(\boldsymbol\theta))]
\right\|\leq M \max_j\left\|\widetilde{\mathbf w}^{(l)}_j\right\|_2 \left\|{\mathbf x}_i^{(l-1)}(\widetilde{\boldsymbol\theta})-{\mathbf x}_i^{(l-1)}(\boldsymbol\theta)\right\|_2\label{diag_diff3}
\end{gather}
In a similar way, the inequality
\begin{align}
\left\|\text{diag} [\phi'(\widetilde{{\mathbf W}}^{(l)} {\mathbf x}_i^{(l-1)}(\boldsymbol\theta))]
-\text{diag} [\phi'({\mathbf W}^{(l)} {\mathbf x}_i^{(l-1)}(\boldsymbol\theta))]
\right\|\leq M \max_j\left\|\widetilde{\mathbf w}^{(l)}_j-
{\mathbf w}^{(l)}_j\right\|_2 
\left\|{\mathbf x}_i^{(l-1)}(\boldsymbol\theta)\right\|_2   
\label{diag_diff4}
\end{align}
can be derived. Inserting \eqref{diag_diff3} and \eqref{diag_diff4}
in \eqref{diag_diff}, we obtain 
\begin{align}
&\left\|
\text{diag} [\phi'(\widetilde{{\mathbf W}}^{(l)} {\mathbf x}_i^{(l-1)}(\widetilde{\boldsymbol\theta}))]
\widetilde{{\mathbf W}}^{(l)}-\text{diag} [\phi'({\mathbf W}^{(l)} {\mathbf x}_i^{(l-1)}(\boldsymbol\theta))]{\mathbf W}^{(l)}\right\|
\nonumber\\
&\,\,\leq B\,\|\widetilde{{\mathbf W}}^{(l)}-
{\mathbf W}^{(l)}\|\nonumber\\&\quad+ M \left(\max_j\left\|\widetilde{\mathbf w}^{(l)}_j\right\|_2 \left\|{\mathbf x}_i^{(l-1)}(\widetilde{\boldsymbol\theta})-{\mathbf x}_i^{(l-1)}(\boldsymbol\theta)\right\|_2
+\max_j\left\|\widetilde{\mathbf w}^{(l)}_j-
{\mathbf w}^{(l)}_j\right\|_2 
\left\|{\mathbf x}_i^{(l-1)}(\boldsymbol\theta)\right\|_2\right)
\|{\mathbf W}^{(l)}\|.
\label{diag_diff5}
\end{align}

There are two terms in \eqref{diag_diff5} depending on the network input ${\mathbf x}_i$, one is $\|{\mathbf x}_i^{(l-1)}(\widetilde{\boldsymbol\theta})-{\mathbf x}_i^{(l-1)}(\boldsymbol\theta)\|_2$ and the other one is $\|{\mathbf x}_i^{(l-1)}(\boldsymbol\theta)\|_2.$
We have already shown 
$$\|{\mathbf x}_i^{(l-1)}\|_2 \leq \sqrt{\frac{c_{\phi}}{m}}B
\|{\mathbf W}^{(1)}\|\,\|{\mathbf x_i}\|_2
\mathlarger{\mathlarger{\prod}}_{j=2}^{l-1}\left(1+
\frac{\crs}{H\sqrt{m}}B \|{\mathbf W}^{(j)}\|\right)$$
in \eqref{frob_upper2}. To upper bound the norm 
$\|{\mathbf x}_i^{(l-1)}(\widetilde{\boldsymbol\theta})-{\mathbf x}_i^{(l-1)}(\boldsymbol\theta)\|_2$, we use mean value theorem for vector valued multi-variable functions and write
\begin{align}
\left\|{\mathbf x}_i^{(l-1)}(\widetilde{\boldsymbol\theta})-{\mathbf x}_i^{(l-1)}(\boldsymbol\theta)\right\|_2 \leq
\left\| \frac{\partial {\mathbf x}_i^{(l-1)}}{\partial \boldsymbol\theta}[\alpha \widetilde{\boldsymbol\theta}+(1-\alpha)
\boldsymbol\theta]\right\|\cdot \left\|\widetilde{\boldsymbol\theta}-\boldsymbol\theta
\right\|_F \label{diag_diff6}
\end{align}
for some number $\alpha$ such that $0\leq \alpha \leq 1.$ 
It follows from \eqref{diag_diff6} that the only remaining quantity we need to consider the norm of is the partial 
$\left\|\frac{\partial{\mathbf x}_i^{(l-1)}}{\partial\boldsymbol\theta}\right\|.$
The norm $\left\|\frac{\partial{\mathbf x}_i^{(l-1)}}{\partial{\mathbf W}^{(h)}}[\boldsymbol\theta]\right\|$ can be 
expressed as
\begin{align}
\left\|\frac{\partial{\mathbf x}_i^{(l-1)}}{\partial{\mathbf W}^{(h)}}[\boldsymbol\theta]\right\|&\leq
\mathlarger{\mathlarger{\prod}}_{j=h+1}^{l-1}\left(1+\frac{\crs}{H\sqrt{m}}B
\|{\mathbf W}^{(j)}\|\right) \left\|\frac{\partial{\mathbf x}_i^{(h)}}{\partial{\mathbf W}^{(h)}}[\boldsymbol\theta]\right\|\nonumber\\
&\leq \begin{cases*}
\mathlarger{\mathlarger{\prod}}_{j=h+1}^{l-1}\left(1+\frac{\crs}{H\sqrt{m}}B
\|{\mathbf W}^{(j)}\|\right)\frac{B\crs}{H\sqrt{m}}
\|{\mathbf x}_i^{(h-1)}(\boldsymbol\theta)\|_2,
&\text{if} $1<h<l$,\\
\mathlarger{\mathlarger{\prod}}_{j=h+1}^{l-1}\left(1+\frac{\crs}{H\sqrt{m}}B
\|{\mathbf W}^{(j)}\|\right)B \sqrt{\frac{c_{\phi}}{m}}
\|{\mathbf x}_i^{(h-1)}(\boldsymbol\theta)\|_2,
&\text{if} $h=1$
\end{cases*}\nonumber
\end{align}

\begin{align}
\left\|\frac{\partial{\mathbf x}_i^{(l-1)}}{\partial{\mathbf W}^{(h)}}[\boldsymbol\theta]\right\|
&\leq \begin{cases*}
\mathlarger{\mathlarger{\prod}}_{j=2}^{l-1}\left(1+\frac{\crs}{H\sqrt{m}}B
\|{\mathbf W}^{(j)}\|\right)
\frac{B^2\crs\sqrt{c_{\phi}}}{H\,m}\|{\mathbf W}^{(1)}\| \|{\mathbf x}_i\|_2 &\text{if} $1<h<l$,\\
\mathlarger{\mathlarger{\prod}}_{j=2}^{l-1}\left(1+\frac{\crs}{H\sqrt{m}}B
\|{\mathbf W}^{(j)}\|\right)B \sqrt{\frac{c_{\phi}}{m}}
\|{\mathbf x}_i\|_2,
&{\text if} $h=1$
\end{cases*}
\label{partial_cases}
\end{align}
using \eqref{meta_partial} and \eqref{frob_upper2}. 
We invoke \eqref{partial_cases} to upper bound 
$\left\| \frac{\partial {\mathbf x}_i^{(l-1)}}{\partial \boldsymbol\theta}[\alpha \widetilde{\boldsymbol\theta}+(1-\alpha)
\boldsymbol\theta]\right\|$ and derive

\begin{align}
&\left\| \frac{\partial {\mathbf x}_i^{(l-1)}}{\partial \boldsymbol\theta}[\alpha \widetilde{\boldsymbol\theta}+(1-\alpha)
\boldsymbol\theta]\right\| = \max_k 
\left\| \frac{\partial {\mathbf x}_i^{(l-1)}}{\partial {\mathbf W}^{(k)}}[\alpha\widetilde{\boldsymbol\theta}+(1-\alpha)\boldsymbol\theta] \right\|
\nonumber\\
&\leq 
\mathlarger{\mathlarger{\prod}}_{j=2}^{l-1}\left(1+\frac{\crs}{H\sqrt{m}}B
\|\alpha 
\widetilde{{\mathbf W}}^{(j)}+(1-\alpha){\mathbf W}^{(j)}\|\right)
B \sqrt{\frac{c_{\phi}}{m}}
\|{\mathbf x}_i\|_2 
\cdot
\max\left\{
\frac{B\crs}{H\sqrt{m}}\|\alpha 
\widetilde{{\mathbf W}}^{(1)}
+(1-\alpha){\mathbf W}^{(1)}\|,\,1
\right\}
\nonumber\\
&\leq \mathlarger{\mathlarger{\prod}}_{j=2}^{l-1}\left(1+\frac{\crs}{H\sqrt{m}}B
\|\alpha 
\widetilde{{\mathbf W}}^{(j)}+(1-\alpha){\mathbf W}^{(j)}\|\right)
B \sqrt{\frac{c_{\phi}}{m}}
\|{\mathbf x}_i\|_2 
\left(1+\frac{\crs}{H\sqrt{m}}B\|\alpha 
\widetilde{{\mathbf W}}^{(1)}
+(1-\alpha){\mathbf W}^{(1)}\| \right)
\nonumber\\
&\leq B \sqrt{\frac{c_{\phi}}{m}}
\|{\mathbf x}_i\|_2 \mathlarger{\mathlarger{\prod}}_{j=1}^{l-1}\left(1+\frac{\crs}{H\sqrt{m}}B
\|\alpha 
\widetilde{{\mathbf W}}^{(j)}+(1-\alpha){\mathbf W}^{(j)}\|\right). \label{partial_cases2}
\end{align}
Then inserting \eqref{partial_cases2} and \eqref{frob_upper2} in \eqref{diag_diff5},
we get
\begin{align}
&\left\|
\text{diag} [\phi'(\widetilde{{\mathbf W}}^{(l)} {\mathbf x}_i^{(l-1)}(\widetilde{\boldsymbol\theta}))]
\widetilde{{\mathbf W}}^{(l)}-\text{diag} [\phi'({\mathbf W}^{(l)} {\mathbf x}_i^{(l-1)}(\boldsymbol\theta))]{\mathbf W}^{(l)}\right\|
\nonumber\\
&\,\,\leq B\,\|\widetilde{{\mathbf W}}^{(l)}-
{\mathbf W}^{(l)}\|\nonumber\\&\quad+ M
\|{\mathbf W}^{(l)}\|
\max_j\left\|\widetilde{\mathbf w}^{(l)}_j\right\|_2 
\|\widetilde{\boldsymbol\theta}-\boldsymbol\theta\|_F
B \sqrt{\frac{c_{\phi}}{m}}
\|{\mathbf x}_i\|_2 \mathlarger{\mathlarger{\prod}}_{k=1}^{l-1}\left(1+\frac{\crs}{H\sqrt{m}}B
\|\alpha 
\widetilde{{\mathbf W}}^{(k)}+(1-\alpha){\mathbf W}^{(k)}\|\right)\nonumber\\
&\quad+M \|{\mathbf W}^{(l)}\| \max_j\left\|\widetilde{\mathbf w}^{(l)}_j-
{\mathbf w}^{(l)}_j\right\|_2 
B\sqrt{\frac{c_{\phi}}{m}}
\|{\mathbf W}^{(1)}\|\,\|{\mathbf x_i}\|_2
\mathlarger{\mathlarger{\prod}}_{k=2}^{l-1}\left(1+
\frac{\crs}{H\sqrt{m}}B \|{\mathbf W}^{(k)}\|\right)
\triangleq f_l(\boldsymbol\theta,\widetilde{\boldsymbol\theta},
{\mathbf x}_i).
\label{diag_diff_final}
\end{align}

The next step would be to combine \eqref{diag_diff_final} 
with \eqref{norm_D_1H}, giving us
\begin{gather}
\|{\mathbf D}_{(1,H)}({\mathbf x}_i)\|_2 \leq   
\left(1+\frac{B\crs}{H\sqrt{m}}
\|\widetilde{{\mathbf W}}^{(H)}\| \right)
\|{\mathbf D}_{(1,H-1)}({\mathbf x}_i)\|_2+ 
\frac{\crs}{H\sqrt{m}}
\mathlarger{\mathlarger{\mathlarger{\prod}}}_{k=2}^{H-1} 
\left(1+\frac{B\crs}{H\sqrt{m}}
\|\widetilde{{\mathbf W}}^{(k)}\|\right)
f_H(\boldsymbol\theta,\widetilde{\boldsymbol\theta},{\mathbf x}_i)
\nonumber\\
\|{\mathbf D}_{(1,H)}({\mathbf x}_i)\|_2 \leq 
\mathlarger{\mathlarger{\mathlarger{\prod}}}_{k=3}^{H}
\left(1+\frac{B\crs}{H\sqrt{m}}
\|\widetilde{{\mathbf W}}^{(k)}\|\right)
\|{\mathbf D}_{(1,2)}({\mathbf x}_i)\|_2\nonumber\\
\quad + 
\frac{\crs}{H\sqrt{m}}
\mathlarger{\mathlarger{\mathlarger{\prod}}}_{k=2}^{H}  
\left(1+\frac{B\crs}{H\sqrt{m}}
\|\widetilde{{\mathbf W}}^{(k)}\|\right)
\mathlarger{\mathlarger{\mathlarger{\sum}}}_{l=3}^H 
f_l(\boldsymbol\theta,\widetilde{\boldsymbol\theta},
{\mathbf x}_i)
\label{recur_simplified}
\end{gather}
Then we evaluate \eqref{norm_D_1H} at $H=2$ to obtain
\begin{align}
\|{\mathbf D}_{(1,2)}({\mathbf x}_i)\|_2 &\leq
\frac{\crs}{H\sqrt{m}}
\left\|
\text{diag} [\phi'(\widetilde{{\mathbf W}}^{(2)} {\mathbf x}_i^{(1)}(\widetilde{\boldsymbol\theta}))]
\widetilde{{\mathbf W}}^{(2)}-\text{diag} [\phi'({\mathbf W}^{(2)} {\mathbf x}_i^{(1)}(\boldsymbol\theta))]{\mathbf W}^{(2)}\right\|
\nonumber\\
&\leq \frac{\crs}{H\sqrt{m}} f_2(\boldsymbol\theta,\widetilde{\boldsymbol\theta},
{\mathbf x}_i)\label{f_2}.
\end{align}
From \eqref{f_2} and \eqref{recur_simplified}, we get the ultimate
upper bound 
\begin{align}
\|{\mathbf D}_{(1,H)}({\mathbf x}_i)\|_2 \leq 
\frac{\crs}{H\sqrt{m}}
\mathlarger{\mathlarger{\mathlarger{\prod}}}_{k=2}^{H}  
\left(1+\frac{B\crs}{H\sqrt{m}}
\|\widetilde{{\mathbf W}}^{(k)}\|\right)
\mathlarger{\mathlarger{\mathlarger{\sum}}}_{l=2}^H 
f_l(\boldsymbol\theta,\widetilde{\boldsymbol\theta},
{\mathbf x}_i)
\label{ultim_bd_D1H}
\end{align}
for $\|{\mathbf D}_{(1,H)}({\mathbf x}_i)\|_2.$

Now let $A$ be a constant such that $\|{\mathbf W}^{(j)}\|\leq A$ and $\|\widetilde{{\mathbf W}}^{(j)}\|\leq A$ for all $j=1,\dots,H.$ Then the variable $f_l(\boldsymbol\theta,\widetilde{\boldsymbol\theta},{\mathbf x}_i)$ given by \eqref{diag_diff_final} can be upper bounded as
\begin{align}
f_l(\boldsymbol\theta,\widetilde{\boldsymbol\theta},{\mathbf x}_i)
\leq B\,\|\widetilde{{\mathbf W}}^{(l)}-
{\mathbf W}^{(l)}\|+ A^2 B M \sqrt{\frac{\cph}{m}}
e^{\frac{AB\crs}{\sqrt{m}}} \|{\mathbf x}_i\|_2 \left(\|\widetilde{\boldsymbol\theta}-\boldsymbol\theta\|_F
+\|\widetilde{{\mathbf W}}^{(l)}- {\mathbf W}^{(l)}\|\right)
\label{f_l_bd}
\end{align}
Then we take the sum of both sides of \eqref{f_l_bd} to get
\begin{align}
\mathlarger{\mathlarger{\mathlarger{\sum}}}_{l=2}^H 
f_l(\boldsymbol\theta,\widetilde{\boldsymbol\theta},
{\mathbf x}_i)&\leq H A^2 B M \sqrt{\frac{\cph}{m}}
e^{\frac{AB\crs}{\sqrt{m}}} \|{\mathbf x}_i\|_2 \|\widetilde{\boldsymbol\theta}-\boldsymbol\theta\|_F
\nonumber\\
&\quad+\left(B+A^2 B M \sqrt{\frac{\cph}{m}}
e^{\frac{AB\crs}{\sqrt{m}}} \|{\mathbf x}_i\|_2\right)
\mathlarger{\mathlarger{\mathlarger{\sum}}}_{l=2}^H 
\|\widetilde{{\mathbf W}}^{(l)}- {\mathbf W}^{(l)}\|
\label{f_l_sum}
\end{align}
and invoke \eqref{ultim_bd_D1H} to obtain
\begin{align}
\|{\mathbf D}_{(1,H)}({\mathbf x}_i)\|_2 &\leq 
A^2 B M \frac{\crs\sqrt{\cph}}{m} e^{\frac{2AB\crs}{\sqrt{m}}} \|{\mathbf x}_i\|_2 \|\widetilde{\boldsymbol\theta}-\boldsymbol\theta\|_F
\nonumber\\&\quad
+\frac{1}{H}
\left(B\frac{\crs}{\sqrt{m}}e^{\frac{AB\crs}{\sqrt{m}}}
+A^2 B M \frac{\crs\sqrt{\cph}}{m}
e^{\frac{2AB\crs}{\sqrt{m}}} \|{\mathbf x}_i\|_2\right)
\mathlarger{\mathlarger{\mathlarger{\sum}}}_{l=2}^H 
\|\widetilde{{\mathbf W}}^{(l)}- {\mathbf W}^{(l)}\|
\label{D1H_bd_A}
\end{align}
Noting that Cauchy-Schwarz inequality gives us $\mathlarger{\mathlarger{\sum}}_{l=1}^H
\|\widetilde{{\mathbf W}}^{(l)}-
{\mathbf W}^{(l)}\| \leq 
\sqrt{H} \left(\mathlarger{\mathlarger{\sum}}_{l=1}^H
\|\widetilde{{\mathbf W}}^{(l)}-
{\mathbf W}^{(l)}\|^2\right)^{\frac{1}{2}}\!\!\!\!\leq
\sqrt{H} \|\widetilde{\boldsymbol\theta}-\boldsymbol\theta\|_F,$
we conclude from \eqref{D1H_bd_A} that
\begin{align}
\|{\mathbf D}_{(1,H)}({\mathbf x}_i)\|_2 &\leq 
\left[\left(1+\frac{1}{\sqrt{H}}\right)
A^2 B M \frac{\crs\sqrt{\cph}}{m}
e^{\frac{2AB\crs}{\sqrt{m}}} \|{\mathbf x}_i\|_2
+\frac{1}{\sqrt{H}}
B\frac{\crs}{\sqrt{m}}e^{\frac{AB\crs}{\sqrt{m}}}
\right] \|\widetilde{\boldsymbol\theta}-\boldsymbol\theta\|_F
\label{D1H_bd_A2}
\end{align}
Then we combine \eqref{v1H2_bd} with \eqref{D1H_bd_A2} to derive
the bound
\begin{align}
\left\|{\mathbf v}_{(1,H)}^{(2)}({\mathbf x}_i)\right\|
&\leq \sqrt{\cph}\|{\mathbf a}\|_{\infty} 
\left[\left(1+\frac{1}{\sqrt{H}}\right)
A^2 B M \frac{\crs\sqrt{\cph}}{m}
e^{\frac{2AB\crs}{\sqrt{m}}} \|{\mathbf x}_i\|_2
+\frac{1}{\sqrt{H}}
B\frac{\crs}{\sqrt{m}}e^{\frac{AB\crs}{\sqrt{m}}}
\right] 
B \|{\mathbf x}_i\|_2 
\nonumber\\&\quad\cdot
\|\widetilde{\boldsymbol\theta}-\boldsymbol\theta\|_F
\label{v1H2_bdA}
\end{align}
for $\left\|{\mathbf v}_{(1,H)}^{(2)}({\mathbf x}_i)\right\|.$
This completes the bound derivation for the norm of the vector
${\mathbf v}_{(1,H)}^{(2)}({\mathbf x}_i).$

Now we return to $\left\|{\mathbf v}_{(1,H)}^{(1)}({\mathbf x}_i)\right\|,$ and use \eqref{v1H1_bd} to get the inequality
\begin{align}
\left\|{\mathbf v}_{(1,H)}^{(1)}({\mathbf x}_i)\right\|_2\leq 
\sqrt{\cph}\|{\mathbf a}\|_{\infty}
e^{\frac{AB\crs}{\sqrt{m}}}
M\|{\mathbf x}_i\|_2 \left\|\widetilde{{\mathbf W}}^{(1)}-{\mathbf W}^{(1)}\right\|_F
\label{v1H1_bdA}
\end{align}
From \eqref{v1H2_bdA} and \eqref{v1H1_bdA}, we derive the
bound
\begin{align}
&\left\|
{\mathbf a}^T \left(\frac{\partial{\mathbf x}_i^{(H)}}{\partial{\mathbf W}^{(1)}}[\widetilde{\boldsymbol\theta}]-\frac{\partial{\mathbf x}_i^{(H)}}{\partial{\mathbf W}^{(1)}}[\boldsymbol\theta]\right)
\right\|_2 \leq \sqrt{\cph}\|{\mathbf a}\|_{\infty} 
\|{\mathbf x}_i\|_2 \|\widetilde{\boldsymbol\theta}-\boldsymbol\theta\|_F
\nonumber\\
&
\cdot
\left[\left(1+\frac{1}{\sqrt{H}}\right)
A^2 B^2 M \frac{\crs\sqrt{\cph}}{m}
e^{\frac{2AB\crs}{\sqrt{m}}} \|{\mathbf x}_i\|_2
+M e^{\frac{AB\crs}{\sqrt{m}}}+
\frac{1}{\sqrt{H}}
B^2\frac{\crs}{\sqrt{m}}e^{\frac{AB\crs}{\sqrt{m}}}
\right] \label{H1_one_before_final}
\end{align}
If $\|{\mathbf x}_i\|=1$ for all $i=1,\dots,n$, then \eqref{H1_one_before_final} implies
\begin{align}
&\sqrt{
\mathlarger{\mathlarger{\sum}}_{i=1}^n
\left\|
{\mathbf a}^T \left(\frac{\partial{\mathbf x}_i^{(H)}}{\partial{\mathbf W}^{(1)}}[\widetilde{\boldsymbol\theta}]-\frac{\partial{\mathbf x}_i^{(H)}}{\partial{\mathbf W}^{(1)}}[\boldsymbol\theta]\right)
\right\|^2_2} \leq
\sqrt{n} \sqrt{\cph}\|{\mathbf a}\|_{\infty} 
 \|\widetilde{\boldsymbol\theta}-\boldsymbol\theta\|_F
\nonumber\\
&\quad 
\cdot
\left[\left(1+\frac{1}{\sqrt{H}}\right)
A^2 B^2 M \frac{\crs\sqrt{\cph}}{m}
e^{\frac{2AB\crs}{\sqrt{m}}} 
+M e^{\frac{AB\crs}{\sqrt{m}}}+
\frac{1}{\sqrt{H}}
B^2\frac{\crs}{\sqrt{m}}e^{\frac{AB\crs}{\sqrt{m}}}
\right] \label{H1_final}
\end{align}

We need to bound the set of norms $\left\|
{\mathbf a}^T \left(\frac{\partial{\mathbf x}_i^{(H)}}{\partial{\mathbf W}^{(h)}}[\widetilde{\boldsymbol\theta}]-\frac{\partial{\mathbf x}_i^{(H)}}{\partial{\mathbf W}^{(h)}}[\boldsymbol\theta]\right)
\right\|_2, h=2,\dots,H$
in a similar way. Recall from 
\eqref{last_partial} that
\begin{align*}
\frac{\partial{\mathbf x}_i^{(H)}}{\partial{\mathbf W}^{(h)}}
[\boldsymbol\theta]
&=
\mathlarger{\mathlarger{\mathlarger{\prod}}}_{l=h+1}^{H} \left[{\mathbf I}_m+\frac{\crs}{H\sqrt{m}}
\text{diag} [\phi'({\mathbf W}^{(l)} {\mathbf x}_i^{(l-1)}(\boldsymbol\theta))]
{\mathbf W}^{(l)}
\right]
\nonumber\\&\quad
\cdot \frac{\crs}{H\sqrt{m}}\,\,
\text{diag} [\phi'({\mathbf W}^{(h)} {\mathbf x}^{(h-1)}_i)]
\begin{bmatrix}
{{\mathbf x}_i^{(h-1)}}^T & \boldsymbol{0} & \boldsymbol{0}\\
\vdots &\ddots & \vdots\\
\boldsymbol{0} & \boldsymbol{0} & {{\mathbf x}_i^{(h-1)}}^T\\
\end{bmatrix}.    
\end{align*}
if $h>1.$ Similarly to \eqref{H1_decomp}, we expand the expression ${\mathbf a}^T\left(\frac{\partial{\mathbf x}_i^{(H)}}{\partial{\mathbf W}^{(h)}}[\widetilde{\boldsymbol\theta}]-\frac{\partial{\mathbf x}_i^{(H)}}{\partial{\mathbf W}^{(h)}}[\boldsymbol\theta]\right)$ as
\begin{align}
{\mathbf a}^T\left(\frac{\partial{\mathbf x}_i^{(H)}}{\partial{\mathbf W}^{(h)}}[\widetilde{\boldsymbol\theta}]-\frac{\partial{\mathbf x}_i^{(H)}}{\partial{\mathbf W}^{(h)}}[\boldsymbol\theta]\right) 
&={\mathbf v}_{(h,H)}^{(1)}({\mathbf x}_i^{(h-1)}(\boldsymbol\theta))
+{\mathbf v}_{(h,H)}^{(2)}({\mathbf x}_i^{(h-1)}(\boldsymbol\theta))
\nonumber\\&\quad
+{\mathbf v}_{(h,H)}^{(3)}({\mathbf x}_i^{(h-1)}(\boldsymbol\theta))
+{\mathbf v}_{(h,H)}^{(4)}
({\mathbf x}_i^{(h-1)}(\widetilde{\boldsymbol\theta})-{\mathbf x}_i^{(h-1)}(\boldsymbol\theta))
\label{hH_expansion}
\end{align}
where
\begin{align}
{\mathbf v}_{(h,H)}^{(1)}({\mathbf x}_i^{(h-1)}(\boldsymbol\theta))   &\triangleq {\mathbf a}^T_{(h,H)}({\mathbf x}_i,\widetilde{\boldsymbol\theta})
\frac{\crs}{H\sqrt{m}}\,\,
\left(\text{diag} [\phi'(\widetilde{{\mathbf W}}^{(h)} 
{\mathbf x}_i^{(h-1)}(\boldsymbol\theta))
]
-\text{diag} 
[\phi'({\mathbf W}^{(h)} {\mathbf x}_i^{(h-1)}(\boldsymbol\theta))]\right)
\nonumber\\
&\quad\cdot
\begin{bmatrix}
{{\mathbf x}_i^{(h-1)}}^T\!\!(\boldsymbol\theta)
& \boldsymbol{0} & \boldsymbol{0}\\
\vdots &\ddots & \vdots\\
\boldsymbol{0} & \boldsymbol{0} &
{{\mathbf x}_i^{(h-1)}}^T\!\!(\boldsymbol\theta) \\
\end{bmatrix}
\end{align}

\begin{align}
{\mathbf v}_{(h,H)}^{(2)}({\mathbf x}_i^{(h-1)}(\boldsymbol\theta))   &\triangleq \left[{\mathbf a}^T_{(h,H)}({\mathbf x}_i,\widetilde{\boldsymbol\theta})-{\mathbf a}^T_{(h,H)}({\mathbf x}_i,\boldsymbol\theta)\right]
\frac{\crs}{H\sqrt{m}} \text{diag} 
[\phi'({\mathbf W}^{(h)} {\mathbf x}_i^{(h-1)}(\boldsymbol\theta))]
\nonumber\\
&\quad\cdot
\begin{bmatrix}
{{\mathbf x}_i^{(h-1)}}^T\!\!(\boldsymbol\theta) & \boldsymbol{0} & \boldsymbol{0}\\
\vdots &\ddots & \vdots\\
\boldsymbol{0} & \boldsymbol{0} &{{\mathbf x}_i^{(h-1)}}^T\!\!(\boldsymbol\theta) \\
\end{bmatrix}
\end{align}

\begin{align}
{\mathbf v}_{(h,H)}^{(3)}({\mathbf x}_i^{(h-1)}(\boldsymbol\theta))   &\triangleq {\mathbf a}^T_{(h,H)}({\mathbf x}_i,\widetilde{\boldsymbol\theta})
\frac{\crs}{H\sqrt{m}}\,\,
\left(\text{diag} [\phi'(\widetilde{{\mathbf W}}^{(h)} 
{\mathbf x}_i^{(h-1)}(\widetilde{\boldsymbol\theta}))
]
-\text{diag} 
[\phi'(\widetilde{{\mathbf W}}^{(h)} {\mathbf x}_i^{(h-1)}(\boldsymbol\theta))]\right)
\nonumber\\
&\quad\cdot
\begin{bmatrix}
{{\mathbf x}_i^{(h-1)}}^T\!\!(\boldsymbol\theta)
& \boldsymbol{0} & \boldsymbol{0}\\
\vdots &\ddots & \vdots\\
\boldsymbol{0} & \boldsymbol{0} &
{{\mathbf x}_i^{(h-1)}}^T\!\!(\boldsymbol\theta) \\
\end{bmatrix}
\end{align}

\begin{align}
{\mathbf v}_{(h,H)}^{(4)}({\mathbf x}_i^{(h-1)}(\boldsymbol\theta))   &\triangleq {\mathbf a}^T_{(h,H)}({\mathbf x}_i,\widetilde{\boldsymbol\theta}) 
\frac{\crs}{H\sqrt{m}} \text{diag} 
[\phi'(\widetilde{{\mathbf W}}^{(h)} {\mathbf x}_i^{(h-1)}(\widetilde{\boldsymbol\theta}))]
\nonumber\\
&\quad\cdot\begin{bmatrix}
{{\mathbf x}_i^{(h-1)}}^T\!\!(\widetilde{\boldsymbol\theta})- 
{{\mathbf x}_i^{(h-1)}}^T\!\!(\boldsymbol\theta)
& \boldsymbol{0} & \boldsymbol{0}\\
\vdots &\ddots & \vdots\\
\boldsymbol{0} & \boldsymbol{0} &{{\mathbf x}_i^{(h-1)}}^T\!\!(\widetilde{\boldsymbol\theta})-{{\mathbf x}_i^{(h-1)}}^T\!\!(\boldsymbol\theta) \\
\end{bmatrix}
\end{align}
and ${\mathbf a}_{(h,H)}({\mathbf x}_i,\boldsymbol\theta)$
is as defined by \eqref{a_1H_defn}.

Note that the definitions of ${\mathbf v}_{(h,H)}^{(1)}({\mathbf x}_i^{(h-1)}(\boldsymbol\theta))$ and ${\mathbf v}_{(h,H)}^{(2)}({\mathbf x}_i^{(h-1)}(\boldsymbol\theta))$
are identical to those of ${\mathbf v}_{(1,H)}^{(1)}({\mathbf x}_i^{(h-1)}(\boldsymbol\theta))$ and ${\mathbf v}_{(1,H)}^{(2)}({\mathbf x}_i^{(h-1)}(\boldsymbol\theta))$
given by \eqref{a_1H_defn}, except for the difference that
the multiplier $\sqrt{\frac{\cph}{m}}$ is replaced by 
$\frac{\crs}{H\sqrt{m}}.$ Similarly to the equations 
\eqref{lemma6_8} and \eqref{diff_infty}, we have
\begin{align}
\left\|
{\mathbf v}_{(h,H)}^{(1)}({\mathbf x}_i^{(h-1)}(\boldsymbol\theta))
\right\|_2 &\leq \frac{\crs}{H\sqrt{m}}
\left\|{\mathbf a}_{(h,H)}
({\mathbf x}_i,\widetilde{\boldsymbol\theta})\right\|_{\infty}
M \|{\mathbf x}_i^{(h-1)}(\boldsymbol\theta)\|_2\left\| \widetilde{{\mathbf W}}^{(h)}-{\mathbf W}^{(h)}\right\|_F
\label{hH1}\\
\left\|
{\mathbf v}_{(h,H)}^{(2)}({\mathbf x}_i^{(h-1)}(\boldsymbol\theta))
\right\|_2 &\leq
\frac{\crs}{H\sqrt{m}}
\left\|{\mathbf a}_{(h,H)}({\mathbf x}_i,\widetilde{\boldsymbol\theta})-{\mathbf a}_{(h,H)}({\mathbf x}_i,\boldsymbol\theta)\right\|_{\infty}B
\|{\mathbf x}_i^{(h-1)}(\boldsymbol\theta)\|_2
\label{hH2}
\end{align}
Using \eqref{frob_upper2} and \eqref{single_infty_bd}, we rewrite
\eqref{hH1} and \eqref{hH2} as

\begin{align}
\left\|
{\mathbf v}_{(h,H)}^{(1)}({\mathbf x}_i^{(h-1)}(\boldsymbol\theta))
\right\|_2 &\leq \frac{\crs}{H} \|{\mathbf a}\|_{\infty}
\mathlarger{\mathlarger{\prod}}_{l=h+1}^{H}
\left(1+\frac{B\crs}{H\sqrt{m}} 
\|\widetilde{{\mathbf W}}^{(l)}\|\right)M
\nonumber\\ &\quad\cdot
\sqrt{\frac{c_{\phi}}{m}}B
\|{\mathbf W}^{(1)}\|\|{\mathbf x_i}\|_2
\mathlarger{\mathlarger{\prod}}_{j=2}^{h-1}\left(1+
\frac{\crs}{H\sqrt{m}}B \|{\mathbf W}^{(j)}\|\right)
\left\| \widetilde{{\mathbf W}}^{(h)}-{\mathbf W}^{(h)}\right\|_F
\label{hH1_contd}
\\
\left\|{\mathbf v}_{(h,H)}^{(2)}
({\mathbf x}_i^{(h-1)}(\boldsymbol\theta))
\right\|_2 &\leq \frac{\crs}{H\sqrt{m}}
\left\|{\mathbf a}_{(h,H)}({\mathbf x}_i,\widetilde{\boldsymbol\theta})-{\mathbf a}_{(h,H)}({\mathbf x}_i,\boldsymbol\theta)\right\|_{\infty}
\nonumber\\&\quad\cdot
\sqrt{\frac{c_{\phi}}{m}}
B^2
 \|{\mathbf W}^{(1)}\|\|{\mathbf x_i}\|_2
\mathlarger{\mathlarger{\prod}}_{j=2}^{h-1}\left(1+
\frac{\crs}{H\sqrt{m}}B \|{\mathbf W}^{(j)}\|\right)
\label{hH2_contd}
\end{align}
In terms of the constant $A$ which is assumed to be
larger than $\|{\mathbf W}^{(j)}\|$ and $\|\widetilde{{\mathbf W}}^{(j)}\|,$ for all $j=1,\dots,H$, we express \eqref{hH1_contd}
as
\begin{align}
\left\|
{\mathbf v}_{(h,H)}^{(1)}({\mathbf x}_i^{(h-1)}(\boldsymbol\theta))
\right\|_2 \leq \frac{\crs \sqrt{\cph}}{H\sqrt{m}}
\|{\mathbf a}\|_{\infty}
A B M
e^{\frac{AB\crs}{\sqrt{m}}}\|{\mathbf x}_i\|_2
\left\| \widetilde{{\mathbf W}}^{(h)}-{\mathbf W}^{(h)}\right\|_F
\label{hH1_final}
\end{align}

The next part of our derivation would be to bound the norm 
$\left\|
{\mathbf v}_{(h,H)}^{(2)}({\mathbf x}_i^{(h-1)}(\boldsymbol\theta))
\right\|_2.$
Similarly to \eqref{v1H2_bd}, the equation ${\mathbf a}_{(h,H)}({\mathbf x}_i,\widetilde{\boldsymbol\theta})-{\mathbf a}_{(h,H)}({\mathbf x}_i,\boldsymbol\theta)={\mathbf a}^T {\mathbf D}_{(h,H)}({\mathbf x}_i)$ is valid. Here 
${\mathbf D}_{(h,H)}({\mathbf x}_i)$ is the difference matrix
defined in \eqref{D_matrix}. Therefore we have the inequality
\begin{align}
\left\|{\mathbf a}_{(h,H)}({\mathbf x}_i,\widetilde{\boldsymbol\theta})-{\mathbf a}_{(h,H)}({\mathbf x}_i,\boldsymbol\theta)\right\|_{\infty} 
\leq \sqrt{m}\|{\mathbf a}\|_{\infty}
\|{\mathbf D}_{(h,H)}({\mathbf x}_i)\|_2,    
\end{align}
and using \eqref{hH2_contd} yields
\begin{align}
\left\|{\mathbf v}_{(h,H)}^{(2)}
({\mathbf x}_i^{(h-1)}(\boldsymbol\theta))
\right\|_2 &\leq \frac{\crs \sqrt{c_{\phi}}}{H\sqrt{m}}
\|{\mathbf a}\|_{\infty} \|{\mathbf D}_{(h,H)}({\mathbf x}_i)\|_2
B^2
\|{\mathbf W}^{(1)}\|\|{\mathbf x_i}\|_2
\mathlarger{\mathlarger{\prod}}_{j=2}^{h-1}\left(1+
\frac{\crs}{H\sqrt{m}}B \|{\mathbf W}^{(j)}\|\right).
\label{DhH}
\end{align}
We need an upper bound for $\|{\mathbf D}_{(h,H)}({\mathbf x}_i)\|_2$ so that we can invoke \eqref{DhH} to bound the norm of ${\mathbf v}_{(h,H)}^{(2)} ({\mathbf x}_i^{(h-1)}(\boldsymbol\theta))$ later on.
But we have already derived an upper bound for $\|{\mathbf D}_{(1,H)}({\mathbf x}_i)\|_2$ in \eqref{ultim_bd_D1H}, and
it can be adapted to the norm $\|{\mathbf D}_{(h,H)}({\mathbf x}_i)\|_2$ as
\begin{align}
\|{\mathbf D}_{(h,H)}({\mathbf x}_i)\|_2 &\leq 
\frac{\crs}{H\sqrt{m}}
\mathlarger{\mathlarger{\mathlarger{\prod}}}_{k=h+1}^{H}  
\left(1+\frac{B\crs}{H\sqrt{m}}
\|\widetilde{{\mathbf W}}^{(k)}\|\right)
\mathlarger{\mathlarger{\mathlarger{\sum}}}_{l=h+1}^H 
f_l(\boldsymbol\theta,\widetilde{\boldsymbol\theta},
{\mathbf x}_i)
\nonumber\\  &\leq
\frac{\crs}{H\sqrt{m}}
\mathlarger{\mathlarger{\mathlarger{\prod}}}_{k=2}^{H}  
\left(1+\frac{B\crs}{H\sqrt{m}}
\|\widetilde{{\mathbf W}}^{(k)}\|\right)
\mathlarger{\mathlarger{\mathlarger{\sum}}}_{l=2}^H 
f_l(\boldsymbol\theta,\widetilde{\boldsymbol\theta},
{\mathbf x}_i)
\label{ultim_bd_DhH}
\end{align}
where $f_l(\boldsymbol\theta,\widetilde{\boldsymbol\theta},
{\mathbf x}_i)$ is as defined by \eqref{diag_diff_final}.
Equation \eqref{ultim_bd_DhH} implies the upper bound we
have derived for $\|{\mathbf D}_{(1,H)}({\mathbf x}_i)\|_2$
is an upper bound for $\|{\mathbf D}_{(h,H)}({\mathbf x}_i)\|_2$
as well. Hence \eqref{D1H_bd_A2} should hold true for
$\|{\mathbf D}_{(h,H)}({\mathbf x}_i)\|_2$, meaning that
\eqref{DhH} gives us

\begin{align}
\left\|{\mathbf v}_{(h,H)}^{(2)}
({\mathbf x}_i^{(h-1)}(\boldsymbol\theta))
\right\|_2 &\leq  
\left[\left(1+\frac{1}{\sqrt{H}}\right)
A^3 B^3 M \frac{\crs\sqrt{\cph}}{m}
e^{\frac{3AB\crs}{\sqrt{m}}} \|{\mathbf x}_i\|^2_2
+\frac{1}{\sqrt{H}}
A B^3\frac{\crs}{\sqrt{m}}e^{\frac{2AB\crs}{\sqrt{m}}}
\|{\mathbf x_i}\|_2
\right]
\nonumber\\&\quad\cdot
\frac{\crs \sqrt{c_{\phi}}}{H\sqrt{m}}
\|{\mathbf a}\|_{\infty} \|\widetilde{\boldsymbol\theta}-\boldsymbol\theta\|_F
\label{hH2_final}
\end{align}

In order to find an upper bound for $\left\|
{\mathbf v}_{(h,H)}^{(3)}({\mathbf x}_i^{(h-1)}(\boldsymbol\theta))
\right\|_2$, we use the inequality
\begin{align}
\left\|
{\mathbf v}_{(h,H)}^{(3)}({\mathbf x}_i^{(h-1)}(\boldsymbol\theta))
\right\|_2 &\leq \frac{\crs}{H\sqrt{m}}
\left\|{\mathbf a}_{(h,H)}
({\mathbf x}_i,\widetilde{\boldsymbol\theta})\right\|_{\infty}
M \left\|{\mathbf x}_i^{(h-1)}(\widetilde{\boldsymbol\theta})-
{\mathbf x}_i^{(h-1)}(\boldsymbol\theta)\right\|_2
\left\| \widetilde{{\mathbf W}}^{(h)}\right\|_2
\label{hH3}
\end{align}
which is quite similar to \eqref{hH1}, and can be easily justified
by the proof of Lemma~6.8 in \cite{oymak2019towards}. From
\eqref{hH3}, we obtain
\begin{align}
\left\|
{\mathbf v}_{(h,H)}^{(3)}({\mathbf x}_i^{(h-1)}(\boldsymbol\theta))
\right\|_2 &\leq \frac{\crs}{H}\|{\mathbf a}\|_{\infty} 
\mathlarger{\mathlarger{\mathlarger{\prod}}}_{k=h+1}^{H}  
\left(1+\frac{B\crs}{H\sqrt{m}}
\|\widetilde{{\mathbf W}}^{(k)}\|\right)
M \left\|{\mathbf x}_i^{(h-1)}(\widetilde{\boldsymbol\theta})-
{\mathbf x}_i^{(h-1)}(\boldsymbol\theta)\right\|_2
\left\| \widetilde{{\mathbf W}}^{(h)}\right\|_2
\end{align}
Then we use \eqref{partial_cases2} to get
\begin{align}
\left\|
{\mathbf v}_{(h,H)}^{(3)}({\mathbf x}_i^{(h-1)}(\boldsymbol\theta))
\right\|_2& \leq \frac{\crs}{H}\|{\mathbf a}\|_{\infty} 
\mathlarger{\mathlarger{\mathlarger{\prod}}}_{k=h+1}^{H}  
\left(1+\frac{B\crs}{H\sqrt{m}}
\|\widetilde{{\mathbf W}}^{(k)}\|\right)
M \nonumber\\ &\cdot
B \sqrt{\frac{c_{\phi}}{m}}
\|{\mathbf x}_i\|_2 \mathlarger{\mathlarger{\prod}}_{k=1}^{h-1}\left(1+\frac{\crs}{H\sqrt{m}}B
\|\alpha 
\widetilde{{\mathbf W}}^{(k)}+(1-\alpha){\mathbf W}^{(k)}\|\right)
\|\widetilde{\boldsymbol\theta}-\boldsymbol\theta\|_F
\left\|\widetilde{{\mathbf W}}^{(h)}\right\|_2
\end{align}

Lastly, we provide an upper bound for the term $\left\|
{\mathbf v}_{(h,H)}^{(4)}({\mathbf x}_i^{(h-1)}(\boldsymbol\theta))
\right\|_2,$ which is given by the inequalities
\begin{align}
\left\|
{\mathbf v}_{(h,H)}^{(4)}({\mathbf x}_i^{(h-1)}(\boldsymbol\theta))
\right\|_2 &\leq \frac{\crs}{H}\|{\mathbf a}\|_{\infty}
\mathlarger{\mathlarger{\mathlarger{\prod}}}_{l=h+1}^{H}  
\left(1+\frac{B\crs}{H\sqrt{m}}
\|\widetilde{{\mathbf W}}^{(l)}\|\right)B
\left\|{\mathbf x}_i^{(h-1)}(\widetilde{\boldsymbol\theta})-
{\mathbf x}_i^{(h-1)}(\boldsymbol\theta)\right\|
\nonumber\\
&\leq \frac{\crs}{H}\|{\mathbf a}\|_{\infty}
\mathlarger{\mathlarger{\mathlarger{\prod}}}_{l=h+1}^{H}  
\left(1+\frac{B\crs}{H\sqrt{m}}
\|\widetilde{{\mathbf W}}^{(l)}\|\right)B
\nonumber\\&\cdot
B \sqrt{\frac{c_{\phi}}{m}}
\|{\mathbf x}_i\|_2 \mathlarger{\mathlarger{\mathlarger{\prod}}}_{l=1}^{h-1}\left(1+\frac{\crs}{H\sqrt{m}}B
\|\alpha 
\widetilde{{\mathbf W}}^{(l)}+(1-\alpha){\mathbf W}^{(l)}\|\right)
\|\widetilde{\boldsymbol\theta}-\boldsymbol\theta\|_F
\end{align}

As we did for $\left\|{\mathbf v}_{(h,H)}^{(1)}({\mathbf x}_i^{(h-1)}(\boldsymbol\theta))\right\|_2$ and $\left\|
{\mathbf v}_{(h,H)}^{(2)}
({\mathbf x}_i^{(h-1)}(\boldsymbol\theta))\right\|_2$, we express $\left\|{\mathbf v}_{(h,H)}^{(3)}({\mathbf x}_i^{(h-1)}(\boldsymbol\theta))\right\|_2$ and $\left\|
{\mathbf v}_{(h,H)}^{(4)}({\mathbf x}_i^{(h-1)}(\boldsymbol\theta))
\right\|_2$ in terms of the upper bound $A$ for
the 2-norm of the weight matrices ${\mathbf W}^{(j)}$
and $\widetilde{{\mathbf W}}^{(j)}$, $j=1,\dots,H$
as

\begin{align}
\left\|
{\mathbf v}_{(h,H)}^{(3)}({\mathbf x}_i^{(h-1)}(\boldsymbol\theta))
\right\|_2 &\leq \frac{\crs\sqrt{\cph}}{H\sqrt{m}}\|{\mathbf a}\|_{\infty}
A B M
e^{\frac{AB\crs}{\sqrt{m}}}
\|{\mathbf x}_i\|_2 
\|\widetilde{\boldsymbol\theta}-\boldsymbol\theta\|_F
\label{hH3_final}
\\
\left\|{\mathbf v}_{(h,H)}^{(4)}({\mathbf x}_i^{(h-1)}(\boldsymbol\theta))
\right\|_2 &\leq
\frac{\crs\sqrt{\cph}}{H\sqrt{m}}\|{\mathbf a}\|_{\infty}
B^2
e^{\frac{AB\crs}{\sqrt{m}}}
\|{\mathbf x}_i\|_2 
\|\widetilde{\boldsymbol\theta}-\boldsymbol\theta\|_F
\label{hH4_final}
\end{align}
We use \eqref{hH1_final}, \eqref{hH2_final},
\eqref{hH3_final} and \eqref{hH4_final} in conjunction with \eqref{hH_expansion} to obtain

\begin{align}
&\left\|{\mathbf a}^T\left(\frac{\partial{\mathbf x}_i^{(H)}}{\partial{\mathbf W}^{(h)}}[\widetilde{\boldsymbol\theta}]-\frac{\partial{\mathbf x}_i^{(H)}}{\partial{\mathbf W}^{(h)}}[\boldsymbol\theta]\right)\right\|_2
\leq 
\frac{\crs \sqrt{\cph}}{H\sqrt{m}} \|{\mathbf a}\|_{\infty}
A B M
e^{\frac{AB\crs}{\sqrt{m}}} \|{\mathbf x}_i\|_2
\left(\left\| \widetilde{{\mathbf W}}^{(h)}-{\mathbf W}^{(h)}\right\|_F +\left\|\widetilde{\boldsymbol\theta}-\boldsymbol\theta\right\|_F \right)
\nonumber\\
&+ \left[\left(1+\frac{1}{\sqrt{H}}\right)
A^3 B^3 M \frac{\crs\sqrt{\cph}}{m}
e^{\frac{3AB\crs}{\sqrt{m}}} \|{\mathbf x}_i\|^2_2
+\frac{1}{\sqrt{H}}
A B^3\frac{\crs}{\sqrt{m}}e^{\frac{2AB\crs}{\sqrt{m}}}
\|{\mathbf x_i}\|_2
\right]
\frac{\crs \sqrt{c_{\phi}}}{H\sqrt{m}}
\|{\mathbf a}\|_{\infty} \|\widetilde{\boldsymbol\theta}-\boldsymbol\theta\|_F
\nonumber\\&
+ \frac{\crs\sqrt{\cph}}{H\sqrt{m}}\|{\mathbf a}\|_{\infty}
B^2
e^{\frac{AB\crs}{\sqrt{m}}}
\|{\mathbf x}_i\|_2 
\|\widetilde{\boldsymbol\theta}-\boldsymbol\theta\|_F
\label{Hh_one_before_final_A}
\end{align}
If $\|{\mathbf x}_i\|=1$ for all $i=1,\dots,n$, then \eqref{Hh_one_before_final_A} implies
\begin{align}
&\sqrt{
\mathlarger{\mathlarger{\sum}}_{i=1}^n
\left\|
{\mathbf a}^T \left(\frac{\partial{\mathbf x}_i^{(H)}}{\partial{\mathbf W}^{(h)}}[\widetilde{\boldsymbol\theta}]-\frac{\partial{\mathbf x}_i^{(H)}}{\partial{\mathbf W}^{(h)}}[\boldsymbol\theta]\right)
\right\|^2_2} 
\nonumber\\
&\quad\leq 
\frac{\crs \sqrt{\cph}}{H\sqrt{m}} \|{\mathbf a}\|_{\infty}
A B M
e^{\frac{AB\crs}{\sqrt{m}}} \sqrt{n}
\left(\left\| \widetilde{{\mathbf W}}^{(h)}-{\mathbf W}^{(h)}\right\|_F +\left\|\widetilde{\boldsymbol\theta}-\boldsymbol\theta\right\|_F \right)
\nonumber\\
&\quad+ \sqrt{n}\left[\left(1+\frac{1}{\sqrt{H}}\right)
A^3 B^3 M \frac{\crs\sqrt{\cph}}{m}
e^{\frac{3AB\crs}{\sqrt{m}}} 
+\frac{1}{\sqrt{H}}
A B^3\frac{\crs}{\sqrt{m}}e^{\frac{2AB\crs}{\sqrt{m}}}
\right]
\frac{\crs \sqrt{c_{\phi}}}{H\sqrt{m}}
\|{\mathbf a}\|_{\infty} \|\widetilde{\boldsymbol\theta}-\boldsymbol\theta\|_F
\nonumber\\&\quad
+ \frac{\crs\sqrt{\cph}}{H\sqrt{m}}\|{\mathbf a}\|_{\infty}
B^2
e^{\frac{AB\crs}{\sqrt{m}}}
\sqrt{n}
\|\widetilde{\boldsymbol\theta}-\boldsymbol\theta\|_F
\label{Hh_final_A}
\end{align}
Then we insert \eqref{Hh_final_A} and \eqref{H1_final} in
\eqref{L_ineq_first} and write
\begin{align}
\|{\mathcal J}(\widetilde{\boldsymbol\theta})-{\mathcal J}
(\boldsymbol\theta)\|\leq 
\sqrt{n}C
\|\widetilde{\boldsymbol\theta}-\boldsymbol\theta\|_F
\end{align}
where
\begin{align}
C &\triangleq  \sqrt{c_{\phi}}
\|{\mathbf a}\|_{\infty} e^{\frac{AB\crs}{\sqrt{m}}}
\left\{M+
\frac{\crs}{\sqrt{m}}
\left[ A B M\left(1+\frac{1}{\sqrt{H}}\right)
+B^2\left(1+\frac{1}{\sqrt{H}}\right)+\frac{1}{\sqrt{H}}A B^3
\frac{\crs}{\sqrt{m}}e^{\frac{AB\crs}{\sqrt{m}}}
\right]\right\}\nonumber\\
&\quad+\frac{\crs c_{\phi}}{m}
\|{\mathbf a}\|_{\infty} e^{\frac{2AB\crs}{\sqrt{m}}}
A^2 B^2 M \left(1+\frac{1}{\sqrt{H}}\right)
\left[1+ \frac{\crs}{\sqrt{m}}A B e^{\frac{AB\crs}{\sqrt{m}}}
\right].\label{L_resultant}
\end{align}
The conclusion of our derivation is that 
if $\|{\mathbf x}_i\|=1$ for all $i=1,\dots,n$,
then we can take the Lipschitz constant for the 
ResNet as 
$
L= C \sqrt{n},
$
where $C$ is as given by \eqref{L_resultant}.

\section{Upper Bound Derivation for Initial Misfit}
\label{misfit_section}

Theorem~\ref{GDthm} is valid under the assumption that
the Jacobian parameters
$\phi_{\min}\left(\mathcal{J}(\vct{\theta})\right)$,
$\|\mathcal{J}(\vct{\theta})\|$ and Lipschitz constant $L$
stays within some predefined limits over the ball of radius
$R=\frac{4\tn{\frs(\bteta_0)-\y}}{\alpha}$
Therefore we need to analyze the error term (a.k.a. initial misfit)   
$\|\frs(\boldsymbol\theta_0)-{\mathbf y}\|$ associated with
the random initialization  $\boldsymbol\theta_0$ of weight parameters so that we can invoke Theorem~\ref{GDthm} to prove the desired convergence result for the ResNet model. Our derivation here relies on the following theorem of \cite{oymak2019towards} proved for one-hidden layer neural networks.

\begin{theorem}[a modified version of Lemma 6.12 in \cite{oymak2019towards}]
Consider a one-hidden layer neural neural network model of the form ${\mathbf x}\to{\mathbf v}^T\phi({\mathbf W}{\mathbf x})$
where the activation $\phi$ has bounded derivatives obeying
$|\phi'(z)|\leq B.$ Also assume we have $n$ data points ${\mathbf x}_1,\dots,{\mathbf x}_n\in{\mathbb R}^d$ \remove{with unit
euclidean norm ($\|{\mathbf x}_i\|_2=1$)} aggregated as rows of a matrix ${\mathbf X}\in{\mathbb R}^{n\times d}$ and the corresponding labels given by ${\mathbf y}\in{\mathbb R}^n.$
Furthermore, assume we set half of the entries of ${\mathbf v}\in{\mathbb R}^k$ to $\frac{\|{\mathbf y}\|_2}{\sqrt{kn}}$
and the other half to $-\frac{\|{\mathbf y}\|_2}{\sqrt{kn}}$.
Then for ${\mathbf W}\in{\mathbb R}^{k\times d}$ with i.i.d.
${\mathcal N}(0,1)$ entries
\begin{equation}
\left\|\phi({\mathbf X} {\mathbf W}^T){\mathbf v}-{\mathbf y}\right\|_2
\leq \|{\mathbf y}\|_2 \left(1+\left(\frac{\|{\mathbf X}\|_F}{\sqrt{n}}+
\delta\right)B\right)\label{twolayer_misfit}
\end{equation}
holds with probability at least 
$1-e^{-\delta^2\frac{n}{2\|{\mathbf X}\|^2}}.$
\label{misfit_lemma}
\end{theorem}

Let the data matrix for the output of $k$-th layer 
associated with the random initialization $\boldsymbol\theta_0$
be denoted as ${\mathbf X}^{(k)}.$ 
It follows from Theorem~\ref{misfit_lemma} that
if \remove{$\|{\mathbf x}_i\|=1,$ for all $i=1,\dots,n$ and}
we sample weight matrix entries with standard
normal distribution, and 
set half of the entries of ${\mathbf a}\in {\mathbb R}^{m}$ to $\frac{\|{\mathbf y}\|_2}{\sqrt{n}}$ and the other half to $-\frac{\|{\mathbf y}\|_2}{\sqrt{n}}$, then we have
\begin{gather}
\left\|\sqrt{\frac{\cph}{m}}\phi({\mathbf X}
{{\mathbf W}^{(1)}}^T){\mathbf a}-\sqrt{\cph}
{\mathbf y}\right\|_2 \leq
\sqrt{\cph} \|{\mathbf y}\|_2 
\left(1+\left(\frac{\|{\mathbf X}\|_F}{\sqrt{n}}+
\delta\right)B\right)
\nonumber\\
\left\|\frac{\crs}{H\sqrt{m}}\phi({\mathbf X}^{(h-1)}
{{\mathbf W}^{(h)}}^T){\mathbf a}-\frac{\crs}{H}
{\mathbf y}\right\|_2 \leq
\frac{\crs}{H}\|{\mathbf y}\|_2 \left(1+
\left(\frac{\|{\mathbf X^{(h-1)}}\|_F}{\sqrt{n}}+
\delta\right)B\right), \quad h=2,\dots,H
\end{gather}
with probability $1-e^{-\delta^2\frac{n}{2\|{\mathbf X}\|^2}}$ and $1-e^{-\delta^2\frac{n}{2\|{\mathbf X}^{(h-1)}\|^2}}, h=2,\dots,H$, respectively. Then a basic induction argument applied with the triangle inequality

\begin{align}
\left\|{\mathbf X}^{(k+1)}{\mathbf a}-
\left(\sqrt{\cph}+\frac{\crs}{H}k\right)
{\mathbf y}\right\|_2
&\leq \left\|{\mathbf X}^{(k)}{\mathbf a}-
\left(\sqrt{\cph}+\frac{\crs}{H}(k-1)
\right){\mathbf y}\right\|_2
\nonumber\\
&\quad+ \left\|\frac{\crs}{H\sqrt{m}}\phi({\mathbf X}^{(k)}
{{\mathbf W}^{(k+1)}}^T){\mathbf a}-\frac{\crs}{H}
{\mathbf y}\right\|_2 
\nonumber\\
&\leq \left\|{\mathbf X}^{(k)}{\mathbf a}-
\left(\sqrt{\cph}+\frac{\crs}{H}(k-1)
\right){\mathbf y}\right\|_2+ 
\frac{\crs}{H}\|{\mathbf y}\|_2 \left(1+
\left(\frac{\|{\mathbf X^{(k)}}\|_F}{\sqrt{n}}+
\delta\right)B\right)
\end{align}
holding true with probability $1-e^{-\delta^2\frac{n}{2\|{\mathbf X}^{(k)}\|^2}}, k=1,\dots,H-1,$
gives us
\begin{align}
\left\|{\mathbf X}^{(h)}{\mathbf a}-
\left(\sqrt{\cph}+\frac{\crs}{H}(h-1)\right){\mathbf y}\right\|
&\leq \left(\sqrt{\cph}+\frac{\crs}{H}(h-1)\right)
\|{\mathbf y}\|_2 (1+\delta B)\nonumber\\
&\quad+\left(\sqrt{\cph}
\frac{\|{\mathbf X}\|_F}{\sqrt{n}}
+\frac{\crs}{H}\sum_{k=1}^{h-1} 
\frac{\|{\mathbf X^{(k)}}\|_F}{\sqrt{n}}
\right)B\|{\mathbf y}\|_2
\label{h_output}
\end{align}
with probability at least $1-e^{-\delta^2\frac{n}{2\|{\mathbf X}\|^2}}-\mathlarger{\sum}_{k=1}^{h-1}e^{-\delta^2\frac{n}{2\|{\mathbf X}^{(k)}\|^2}},$
for all $h=1,\dots,H$. For $h=H$, \eqref{h_output} reduces to the inequality
\begin{align}
\left\|\frs
(\boldsymbol\theta_0)-\left(\sqrt{\cph}+
\frac{H-1}{H}\crs\right)
{\mathbf y}\right\| &\leq 
\left(\sqrt{\cph}+\frac{\crs}{H}(H-1)\right)
\|{\mathbf y}\|_2 (1+\delta B)
\nonumber\\
&\quad+\left(\sqrt{\cph}
\frac{\|{\mathbf X}\|_F}{\sqrt{n}}
+\frac{\crs}{H}\sum_{k=1}^{H-1} 
\frac{\|{\mathbf X^{(k)}}\|_F}{\sqrt{n}}
\right)B\|{\mathbf y}\|_2
\label{misfit_meta}
\end{align}
holding true with probability  $1-e^{-\delta^2\frac{n}{2\|{\mathbf X}\|^2}}-\mathlarger{\sum}_{k=1}^{H-1}e^{-\delta^2\frac{n}{2\|{\mathbf X}^{(k)}\|^2}}.$ Equation \eqref{misfit_meta} implies
\begin{gather}
\|\frs(\boldsymbol\theta_0)\| \leq \left[(\sqrt{\cph}+\crs) 
 (2+\delta B)+\left(\sqrt{\cph}
\frac{\|{\mathbf X}\|_F}{\sqrt{n}}
+\frac{\crs}{H}\sum_{k=1}^{H-1} 
\frac{\|{\mathbf X^{(k)}}\|_F}{\sqrt{n}}
\right)B  \right] \|{\mathbf y}\|_2 \nonumber\\
\|\frs(\boldsymbol\theta_0)-{\mathbf y}\| \leq
\kappa \|{\mathbf y}\|_2\label{misfit_final}
\end{gather}
with probability $1-e^{-\delta^2\frac{n}{2\|{\mathbf X}\|^2}}-\mathlarger{\sum}_{k=1}^{H-1}e^{-\delta^2\frac{n}{2\|{\mathbf X}^{(k)}\|^2}},$ where
\begin{equation}
\kappa\triangleq  
1+(\sqrt{\cph}+\crs) 
(2+\delta B)+\left(\sqrt{\cph}
\frac{\|{\mathbf X}\|_F}{\sqrt{n}}
+\frac{\crs}{H}\sum_{k=1}^{H-1} 
\frac{\|{\mathbf X^{(k)}}\|_F}{\sqrt{n}}
\right)B. \label{kappa}
\end{equation}
Note that $\kappa$ we define here is a finite constant which does not grow with $n$ or $H$, as the upper bound \eqref{frob_upper2} implies.

This completes the derivation of upper bound for the misfit term $\|\frs(\boldsymbol\theta_0)-{\mathbf y}\|.$
Comparing \eqref{misfit_final} with
\eqref{twolayer_misfit}, we observe that
the initial misfit for the ResNet is in the order of 
the label norm $\|{\mathbf y}\|$, similarly to the one-hidden layer network case.

\section{Variation Analysis of $\alpha$, $\beta$ and $L$}
\label{param_variation}

In Section~\ref{jacob_analysis}, we have derived 
a lower bound on $\phi_{\min}({\mathcal J}(\boldsymbol\theta))$ for the random initialization $\boldsymbol\theta_0,$ together with upper bounds on $\|{\mathcal J}(\boldsymbol\theta)\|$ and 
the Lipschitz parameter $\|{\mathcal J}(\widetilde{\boldsymbol\theta})-{\mathcal J}(\boldsymbol\theta)\|/\|\widetilde{\boldsymbol\theta}-
\boldsymbol\theta\|_F$
for a given set $\boldsymbol\theta$ and $\widetilde{\boldsymbol\theta}$ of weight matrices. But
this is not sufficient to use Theorem~\ref{GDthm} as it requires those bounds to be valid over a certain region of parameter space,
rather than for a single parameter set $\boldsymbol\theta$
(or a single pair $\boldsymbol\theta$, $\widetilde{\boldsymbol\theta}$ for the Lipschitz constant).
In this section our goal is to extend the $\alpha$, $\beta$
and $L$ analysis carried out in Section~\ref{jacob_analysis}
to a region of weight matrices,  which is taken to be a ball of radius
$R=\frac{4\|\frs(\boldsymbol\theta_0)-{\mathbf y}\|_2}{\alpha}$ centered at $\boldsymbol\theta_0,$
so that Theorem~\ref{GDthm} is applicable for the ResNet.

Let half of the entries of ${\mathbf a}\in{\mathbb R}^m$
be $\frac{\|{\mathbf y}\|_2}{\sqrt{n}}$ and the other half be $-\frac{\|{\mathbf y}\|_2}{\sqrt{n}},$ 
and assume $\|{\mathbf x}_i\|=1,$ for all $i=1,\dots,n,$
to be consistent with the setup of Section \ref{misfit_section}. 
Recall from \eqref{alpha_final} that we have
$\phi_{\min}[{\mathcal J}(\boldsymbol\theta_0)]
\geq\alpha_0=(1-\delta')\sqrt{\frac{c_{\phi}}{m}} 
\|{\mathbf a}\|_2 e^{-2B \crs} 
\sqrt{\lambda({\mathbf X})},$ 
as long as $0<\delta'<1$ and $H=H(\delta',B,\crs)$ is 
sufficiently large.
Inside the ball
${\mathcal B}\left({\boldsymbol\theta}_0, \frac{4\twonorm{f({\boldsymbol\theta}_0)-{\mathbf y}}}{\alpha_0}\right)$, assume the inequality
$\phi_{\min}[{\mathcal J}(\boldsymbol\theta)]
\geq (1-\delta')\alpha_0$ is satisfied. 
From \eqref{misfit_final} and \eqref{alpha_final}, an upper bound for $R=\frac{4\|\frs(\boldsymbol\theta_0)
-{\mathbf y}\|_2}{\alpha}$ can be written as

\begin{gather}
R\leq 
\frac{4\kappa
\|{\mathbf y}\|_2}{(1-\delta')^2\sqrt{\frac{c_{\phi}}{m}}
\|{\mathbf a}\|_2 e^{-2B \crs} \sqrt{\lambda({\mathbf X})}}= \frac{4\kappa \|{\mathbf y}\|_2}{(1-\delta')^2\sqrt{c_{\phi}} 
\frac{\|{\mathbf y}\|_2}{\sqrt{n}} e^{-2 B \crs} \sqrt{\lambda({\mathbf X})}}
\nonumber\\
R\leq \frac{4\kappa}{(1-\delta')^2\sqrt{c_{\phi}}}
\frac{e^{2 B \crs}}{\sqrt{\lambda({\mathbf X})}}\sqrt{n}
\triangleq R_{\delta,\delta'}
\label{r_delta}
\end{gather}
where $\kappa$ is as defined by \eqref{kappa}.

If we can show $\phi_{\min}[{\mathcal J}(\boldsymbol\theta)]
\geq (1-\delta')\alpha_0$ over the ball ${\mathcal B}\left({\boldsymbol\theta}_0,R_{\delta,\delta'}\right)$, then
$\alpha$ in Theorem~\ref{GDthm} can be taken as
\begin{align}
\alpha_{\delta'}=(1-\delta')\alpha_0=
(1-\delta')^2\sqrt{\frac{c_{\phi}}{m}} 
\|{\mathbf a}\|_2 e^{-2B \crs} \sqrt{\lambda({\mathbf X})}.\label{target_alpha}
\end{align}

\begin{lemma}
If the network width $m$ of the ResNet satisfies the inequalities
\begin{align}
m\geq \frac{64\kappa^2 B^2 \crs^2
e^{4 B \crs}}{(1-\delta')^4
\left(\ln\frac{1}{1-\delta'}\right)^2
c_{\phi}\lambda({\mathbf X})} \,n 
\label{m_cond}    
\end{align}
and
\begin{align}
m\geq
\frac{32\kappa^2 e^{4B\crs}}{d\delta'^2(1-\delta')^4\cph \lambda({\mathbf X})}n,
\label{totally_new_constraint2}
\end{align}
and $H=H(\delta',B,\crs)$ is sufficiently large, then \eqref{target_alpha} follows, i.e., the inequality $\phi_{\min}[{\mathcal J}(\boldsymbol\theta)]
\geq (1-\delta')\alpha_0$ is valid for $\boldsymbol\theta\in{\mathcal B}\left({\boldsymbol\theta}_0,R_{\delta,\delta'}\right)$.
\end{lemma}
\begin{proof}
Let $\boldsymbol\theta\in {\mathcal B}\left({\boldsymbol\theta}_0,R_{\delta,\delta'}\right).$ For such a $\boldsymbol\theta,$ let each weight matrix belonging to it be written as 
${\mathbf W}^{(h)}={\mathbf W}_0^{(h)}+
{\mathbf E}^{(h)}$ for all $h=1,\dots,H,$ where
${\mathbf W}_0^{(h)}$ refers to the random weight matrix
associated with the random initialization and 
${\mathbf E}^{(h)}$
refers to the additional deterministic term for the $h-$th layer. 
Those additional weight matrix terms have to satisfy the constraint
\begin{align}
\mathlarger{\mathlarger{\sum}}_{h=1}^H \|{\mathbf E}^{(h)}\|^2_F \leq 
R_{\delta,\delta'}^2 \label{r_delta_circle}
\end{align}
since the ball $\boldsymbol\theta$ is in assumed to be of radius 
$R_{\delta,\delta'}.$
Going back to Section~\ref{alpha_derivation}, we see that 
the left hand side of \eqref{norm_ineq3} needs to be updated as
$$\|{\mathbf a}\|^2_2 \mathlarger{\mathlarger{\mathlarger{\prod}}}_{j=2}^H
\left(1-\frac{B \crs}{H \sqrt{m}}{\mathbb E}
\|{\mathbf W}_0^{(j)}+{\mathbf E}^{(j)}\|\right)^2$$
to account for the non-random matrices ${\mathbf E}^{(1)},\dots,
{\mathbf E}^{(H)}$ we consider here. To find a lower bound for this updated term, we write
\begin{align}
\|{\mathbf a}\|^2_2 
\mathlarger{\mathlarger{\mathlarger{\prod}}}_{j=2}^H
\left(1-\frac{B \crs}{H \sqrt{m}}{\mathbb E}
\|{\mathbf W}_0^{(j)}+{\mathbf E}^{(j)}\|\right)^2
&\geq \|{\mathbf a}\|^2_2 \mathlarger{\mathlarger{\mathlarger{\prod}}}_{j=2}^H
\left(1-\frac{B \crs}{H \sqrt{m}}({\mathbb E}
\|{\mathbf W}_0^{(j)}\|+\|{\mathbf E}^{(j)}\|)\right)^2
\nonumber\\
\geq \|{\mathbf a}\|^2_2 \mathlarger{\mathlarger{\mathlarger{\prod}}}_{j=2}^H
\left[1-\frac{B \crs}{H} 
\left(2+ \frac{\|{\mathbf E}^{(j)}\|}{\sqrt{m}}
\right)\right]^2
&\geq \|{\mathbf a}\|^2_2 \mathlarger{\mathlarger{\mathlarger{\prod}}}_{j=2}^H
\left[1-\frac{B \crs}{H} 
\left(2+ \frac{\|{\mathbf E}^{(j)}\|_F}{\sqrt{m}}
\right)\right]^2 \label{norm_ineq4}
\end{align}
where we use the inequality $\frac{{\mathbb E} \|{\mathbf W}_0^{(j)}\|}{\sqrt{m}} \leq 2$ in \eqref{norm_ineq4}.
Then we combine \eqref{r_delta_circle} with
\eqref{norm_ineq4} to get
\begin{align}
\|{\mathbf a}\|^2_2 
\mathlarger{\mathlarger{\mathlarger{\prod}}}_{j=2}^H
\left(1-\frac{B \crs}{H \sqrt{m}}{\mathbb E}
\|{\mathbf W}_0^{(j)}+{\mathbf E}^{(j)}\|\right)^2
&\geq \|{\mathbf a}\|^2_2 \mathlarger{\mathlarger{\mathlarger{\prod}}}_{j=2}^H
\left[1-\frac{B \crs}{H} 
\left(2+ \frac{R_{\delta,\delta'}}{\sqrt{m}}\right)
\right]^2\nonumber\\
&= \|{\mathbf a}\|^2_2
\left[1-\frac{B \crs}{H} 
\left(2+ \frac{R_{\delta,\delta'}}{\sqrt{m}}\right)
\right]^{2(H-1)}
\nonumber\\
&\geq \|{\mathbf a}\|^2_2 \sqrt{1-\delta'}
e^{-2\left(2+ \frac{R_{\delta,\delta'}}{\sqrt{m}}\right) B \crs}
\label{norm_ineq5}
\end{align}
for $H$ sufficiently large. Additionally, if the inequality
\begin{align}
e^{-\left(2+ \frac{R_{\delta,\delta'}}{\sqrt{m}}\right) B \crs}   
\geq \sqrt{1-\delta'} e^{-2 B\crs} \label{cond_for_alpha}
\end{align}
is satisfied along with the inequality 
\begin{align}
R_{\delta,\delta'}^2 \leq \frac{\delta'^{2} md}{2} \label{totally_new_constraint} 
\end{align}
so that the deviation of the matrices
\begin{align}
\left(\phi'({\mathbf X}({\mathbf w}_l+{\mathbf e}_l))
\phi'({\mathbf X}({\mathbf w}_l+{\mathbf e}_l))^T\right) \odot
({\mathbf X}{\mathbf X}^T), l=1,\dots,m
\end{align}
from the matrices
$
\left(\phi'({\mathbf X}{\mathbf w}_l)
\phi'({\mathbf X}{\mathbf w}_l)^T\right) \odot
({\mathbf X}{\mathbf X}^T), l=1,\dots,m
$
is controlled (here ${\mathbf e}_1,\dots,{\mathbf e}_m$ correspond to the rows of ${\mathbf E}^{(1)}$)
and \remove{using \eqref{g1_expansion}} we can write
\remove{
\begin{align}
\phi_{\min}
({\mathbf G}^{(1)})&=\phi_{\min}\left[\frac{c_{\phi}}{m} \mathlarger{\mathlarger{\sum}}_{l=1}^m ({\mathbf a}_l {\mathbf a}^T_l)\odot \left[(\phi'({\mathbf X}({\mathbf w}_l+{\mathbf e}_l)) \phi'({\mathbf X}({\mathbf w}_l+{\mathbf e}_l))^T) \odot({\mathbf X}{\mathbf X}^T)\right]
\right] \nonumber\\
&\geq \min\Bigg\{\phi_{\min}\Bigg[\frac{c_{\phi}}{m}\Bigg(
\!\!\!\!
\sum_{\substack{l\in{\mathcal M}\\|{\mathcal M}|\geq m(
1-2R_{\delta,\delta'}^2/md )
}} \!\!\!\!\!\!\!
({\mathbf a}_l {\mathbf a}^T_l)\odot \left[(\phi'({\mathbf X}{\mathbf w}_l) \phi'({\mathbf X} {\mathbf w}_l)^T \odot({\mathbf X}{\mathbf X}^T)\right]\Bigg)\Bigg],
\nonumber\\& \quad
\phi_{\min}\Bigg[\frac{c_{\phi}}{m}\Bigg(
\sum_{i=1}^m 
({\mathbf a}_l {\mathbf a}^T_l)\odot (1-\delta')^2 \left[(\phi'({\mathbf X}{\mathbf w}_l) \phi'({\mathbf X} {\mathbf w}_l)^T \odot({\mathbf X}{\mathbf X}^T)\right]\Bigg)\Bigg]
\Bigg\}
\\
&\geq\min\Bigg\{ \phi_{\min}\Bigg[\frac{c_{\phi}}{m}\Bigg(
\sum_{\substack{l\in{\mathcal M}\\|{\mathcal M}|= m(1-\delta'^2)}} 
({\mathbf a}_l {\mathbf a}^T_l)\odot \left[(\phi'({\mathbf X}{\mathbf w}_l) \phi'({\mathbf X} {\mathbf w}_l)^T \odot({\mathbf X}{\mathbf X}^T)\right]\Bigg)\Bigg],\nonumber\\&\quad
(1-\delta')^2
\phi_{\min}\Bigg[\frac{c_{\phi}}{m}\Bigg(
\sum_{i=1}^m 
({\mathbf a}_l {\mathbf a}^T_l)\odot
\left[(\phi'({\mathbf X}{\mathbf w}_l) \phi'({\mathbf X} {\mathbf w}_l)^T \odot({\mathbf X}{\mathbf X}^T)\right]\Bigg)\Bigg]
\Bigg\}
\end{align}
}
\begin{align}
&\phi_{\min}\left[{\mathbb E}\left\{\frac{\cph}{m}
\mathlarger{\mathlarger{\sum}}_{l=1}^m
\phi'({\mathbf X}({\mathbf w}_l+{\mathbf e}_l)) \phi'({\mathbf X}({\mathbf w}_l+{\mathbf e}_l))^T
\right\}\odot
({\mathbf X}{\mathbf X}^T)\right]
\nonumber\\&\quad\geq
\phi_{\min}\left[{\mathbb E}\left\{\frac{\cph}{m}
\mathlarger{\mathlarger{\sum}}_{l=1}^m
\phi'({\mathbf X}({\mathbf w}_l+{\mathbf e}_l)) \phi'({\mathbf X}({\mathbf w}_l+{\mathbf e}_l))^T
\right\}\right]
\nonumber\\&\quad\geq
\min\Bigg\{
\phi_{\min}\Bigg[{\mathbb E}\Bigg\{\frac{\cph}{m}
\!\!\!\!\!\!\!
\mathlarger{\mathlarger{\sum}}_{\substack{l\in{\mathcal M}\\|{\mathcal M}|\geq m(1-2R_{\delta,\delta'}^2/md )}}
\!\!\!\!\!\!\!
\phi'({\mathbf X}{\mathbf w}_l) \phi'({\mathbf X} {\mathbf w}_l)^T
\Bigg\}\Bigg]\nonumber\\&\quad\quad\quad,
\phi_{\min}\left[{\mathbb E}\left\{\frac{\cph}{m}
\mathlarger{\mathlarger{\sum}}_{l=1}^m
\phi'({\mathbf X}(1-\delta'){\mathbf w}_l) \phi'({\mathbf X} (1-\delta'){\mathbf w}_l)^T
\right\}\right]\Bigg\}
\nonumber\\ &\quad\geq
\min\Bigg\{
\phi_{\min}\Bigg[{\mathbb E}\Bigg\{\frac{\cph}{m}\!\!\!\!
\mathlarger{\mathlarger{\sum}}_{\substack{l\in{\mathcal M}\\|{\mathcal M}|= m(1-\delta'^2)}}\!\!\!\!
\phi'({\mathbf X}{\mathbf w}_l) \phi'({\mathbf X} {\mathbf w}_l)^T
\Bigg\}\Bigg]\nonumber\\&\quad\quad\quad,
(1-\delta')^2
\phi_{\min}\left[{\mathbb E}\left\{\frac{\cph}{m}
\mathlarger{\mathlarger{\sum}}_{l=1}^m 
\phi'({\mathbf X}{\mathbf w}_l) \phi'({\mathbf X}{\mathbf w}_l)^T
\right\}\right]\Bigg\}
\end{align}
for some subset ${\mathcal M}$ of $\{1,\dots,m\},$
then repeating the derivation steps of \eqref{alpha_final_ineq} gives us
\begin{align*}
\phi_{\min}[{\mathcal J}(\boldsymbol\theta)]
\geq \left(\phi_{\min}({\mathbf G}^{(1)})\right)^{\frac{1}{2}}
\geq 
(1-\delta')^2\sqrt{\frac{c_{\phi}}{m}} 
\|{\mathbf a}\|_2 e^{-2B \crs} \sqrt{\lambda({\mathbf X})}
=(1-\delta')\alpha_0,
\end{align*}
with a high probability converging to $1$ as $m\to\infty$. \eqref{cond_for_alpha} is equivalent to
\begin{gather}
\frac{1}{2}\ln \left(\frac{1}{1-\delta'}\right) \geq B\crs \frac{R_{\delta,\delta'}}{\sqrt{m}} \nonumber\\
(1-\delta')^2 \ln\left(\frac{1}{1-\delta'}\right)
\geq \frac{ 8 \kappa B\crs}{\sqrt{c_{\phi}}}
\frac{e^{2 B \crs}}{\sqrt{\lambda({\mathbf X})}} \sqrt{\frac{n}{m}}. \label{cond_for_alpha1}
\end{gather}
\remove{
where 
\begin{align}
K\triangleq \frac{4\left[(\sqrt{\cph}+\crs) (2+(1+\delta)B)+1\right]}{\sqrt{c_{\phi}}}
\frac{e^{2 B \crs}}{\sqrt{\lambda({\mathbf X})}}   
\end{align}
}

Equation \eqref{cond_for_alpha1} can be rewritten as 
\eqref{m_cond},
and equation \eqref{totally_new_constraint} can be rewritten as
\eqref{totally_new_constraint2}.
As long as \eqref{m_cond} holds true, the inequality
\eqref{cond_for_alpha} is true. If
$H=H(\delta',B,\crs)$ is sufficiently large, then \eqref{norm_ineq5} is also true. 
At the same time if \eqref{totally_new_constraint2} is valid,
then $\alpha=(1-\delta')\alpha_0$ follows, i.e., we have shown \eqref{target_alpha}.
\end{proof}
\remove{
Note that the problem given by
\begin{gather}
\text{minimize} \quad \mathlarger{\mathlarger{\mathlarger{\prod}}}_{j=2}^H
\left[1-\frac{B \crs}{H} 
\left(2+ \frac{\|{\mathbf E}^{(j)}\|_F}{\sqrt{m}}
\right)\right]^2 \nonumber\\
\text{subject}\,\text{to} \quad 
\mathlarger{\mathlarger{\sum}}_{h=1}^H \|{\mathbf E}^{(h)}\|^2_F \leq  R_{\delta}^2
\end{gather}
is a concave one. Therefore its minimum has to be achieved at a
corner point such that $\|{\mathbf E}^{(j)}\|_F=R_{\delta}$ for some $j$, and ${\mathbf E}^{(j')}=0$ if $j'\neq j.$ This implies
\begin{align}
\mathlarger{\mathlarger{\mathlarger{\prod}}}_{j=2}^H
\left[1-\frac{B \crs}{H} 
\left(2+ \frac{\|{\mathbf E}^{(j)}\|_F}{\sqrt{m}}
\right)\right]^2 &\geq \left[1-\frac{B\crs}{H}\left(2+\frac{R_{\delta}}{\sqrt{m}}\right)
\right]^2 \mathlarger{\mathlarger{\mathlarger{\prod}}}_{j=3}^H
\left(1-2\frac{B \crs}{H} \right)\nonumber\\
&\geq \left[1-\frac{B\crs}{H}\left(2+\frac{R_{\delta}}{\sqrt{m}}\right)
\right]^2 (1-\delta)e^{-4B\crs}
\end{align}
for $H$ sufficiently large. Additionally, if the inequality
\begin{align}
1-\frac{B\crs}{H}\left(2+\frac{R_{\delta}}{\sqrt{m}}
\right) \geq 1-\delta  \label{cond_for_alpha}
\end{align}
is satisfied, then repeating the derivation steps of \eqref{alpha_final_ineq} give us
\begin{align*}
\phi_{\min}[{\mathcal J}(\boldsymbol\theta)]
\geq (1-\delta)\alpha_0 =
(1-\delta)^2\sqrt{\frac{c_{\phi}}{m}} 
\|{\mathbf a}\|_2 e^{-2B \crs} \sqrt{\lambda({\mathbf X})},
\end{align*}
as desired. \eqref{cond_for_alpha} is equivalent to
}

The next step would be to consider how much the Jacobian upper bound $\beta$ varies over the ball ${\mathcal B}\left({\boldsymbol\theta}_0,R_{\delta,\delta'}\right).$ Recall from \eqref{beta_final} that for a given $\boldsymbol\theta$, we have 
\begin{align*}
\beta=\|{\mathbf a}\|_{2}\left(B\sqrt{\frac{\cph}{m}}\,\,
+A\,B^2 
\frac{\sqrt{c_{\phi}}\crs}{\sqrt{H} m}\right) 
e^{\frac{AB\crs}{\sqrt{m}}}
\|{\mathbf X}\|_F    
\end{align*}
where $A$ refers to the upper bound satisfying
$\|{\mathbf W}^{(j)}\|,
\|\widetilde{{\mathbf W}}^{(j)}\|\leq A, j=1,\dots,H.$
Note that half of the entries of ${\mathbf a}$ is chosen as
$\frac{\|{\mathbf y}\|_2}{\sqrt{n}}$ and the other half as
$-\frac{\|{\mathbf y}\|_2}{\sqrt{n}}$, so 
$\|{\mathbf a}\|_2=\frac{\|{\mathbf y}\|_2\sqrt{m}}{\sqrt{n}}.$ Moreover,
$\|{\mathbf x}_i\|_2=1$ for all $i=1,\dots,n$ implies
$\|{\mathbf X}\|_F=\sqrt{n}.$
In this case, \eqref{beta_final} yields 
\begin{equation}
\beta= \|{\mathbf y}\|_{2}
\left(B\sqrt{\cph}\,\,
+A\,B^2 
\frac{\sqrt{c_{\phi}}\crs}{\sqrt{H\,m}}\right) 
e^{\frac{AB\crs}{\sqrt{m}}}
\label{beta_variation}
\end{equation}
If $\boldsymbol\theta \in {\mathcal B}\left({\boldsymbol\theta}_0,R_{\delta,\delta'}\right),$
then the weight matrices are of the form
${\mathbf W}^{(h)}={\mathbf W}_0^{(h)}+
{\mathbf E}^{(h)}, h=1,\dots,H,$ as discussed above. 
Using  \eqref{r_delta_circle} and concentration inequality for the random Gaussian matrix norms \cite{rudelson2010non}, we get
\begin{align}
\|{\mathbf W}^{(h)}\|\leq \|{\mathbf W}_0^{(h)}\|
+\|{\mathbf E}^{(h)}\|
\leq 3\sqrt{m}+ \|{\mathbf E}^{(h)}\|
\leq 3\sqrt{m}+ \|{\mathbf E}^{(h)}\|_F
\leq  3\sqrt{m}+R_{\delta,\delta'} \label{Wh_bound}
\end{align}
with a probability $1-e^{-m^2/2}.$
Under the condition given by \eqref{cond_for_alpha1}, 
we can write another bound for $\|{\mathbf W}^{(h)}\|$ as
\begin{align}
\|{\mathbf W}^{(h)}\|\leq 3\sqrt{m}+\frac{1}{2B\crs}\ln\frac{1}{1-\delta'}\sqrt{m}=
\left[3+\frac{1}{2B\crs}\ln\frac{1}{1-\delta'}\right]\sqrt{m}
\label{Wh_bound_2}
\end{align}
for all $h=1,\dots,H$ with a probability $1-H\,e^{-m^2/2}.$ Therefore for all 
$\boldsymbol\theta\in 
{\mathcal B}\left({\boldsymbol\theta}_0,R_{\delta,\delta'}\right)$
we can take $A$ appearing in \eqref{beta_variation} to be
\begin{equation}
A=\left[3+\frac{1}{2B\crs}\ln\frac{1}{1-\delta'}\right]\sqrt{m},
\label{A_final}
\end{equation}
which gives us the upper bound
\begin{align}
\beta_{\delta'}&= \|{\mathbf y}\|_{2} 
\left[B\sqrt{\cph}
+B^2 \frac{\sqrt{c_{\phi}}\crs}{\sqrt{H}}
\left(3+\frac{1}{2B\crs}\ln\frac{1}{1-\delta'}\right)
\right] 
e^{\left(3+\frac{1}{2B\crs}\ln\frac{1}{1-\delta'}\right)B\crs}
\nonumber\\
&= \frac{\|{\mathbf y}\|_{2}}{\sqrt{1-\delta'}} 
\left[B\sqrt{\cph}
+B^2 \frac{\sqrt{c_{\phi}}\crs}{\sqrt{H}}
\left(3+\frac{1}{2B\crs}\ln\frac{1}{1-\delta'}\right)
\right]e^{3B\crs}
\label{beta_delta}
\end{align}
for $\beta$ over the ball ${\mathcal B}\left({\boldsymbol\theta}_0,R_{\delta,\delta'}\right).$

The last parameter we need to consider is the Lipschitz constant $L.$
Combining \eqref{A_final} with \eqref{L_resultant}, 
and using $\|{\mathbf a}\|_{\infty}=
\frac{\|{\mathbf y}\|_2}{\sqrt{n}},$
we derive the Lipschitz constant $L_{\delta'}$
over the ball ${\mathcal B}\left({\boldsymbol\theta}_0,R_{\delta,\delta'}\right)$
as
\begin{align}
L_{\delta'} &\triangleq  \sqrt{c_{\phi}}
\|{\mathbf y}\|_2 \frac{e^{3B\crs}}{\sqrt{1-\delta'}}
\Bigg[ M+
\crs
\left(3+\frac{1}{2B\crs}\ln\frac{1}{1-\delta'}\right) B M\left(1+\frac{1}{\sqrt{H}}\right)
+\crs B^2\left(1+\frac{1}{\sqrt{H}}\right)
\nonumber\\
&\quad+
\crs\frac{1}{\sqrt{H}} \left(3+\frac{1}{2B\crs}\ln\frac{1}{1-\delta'}\right)B^3
\crs \frac{e^{3B\crs}}{\sqrt{1-\delta'}}
\Bigg]+ \crs c_{\phi}
\|{\mathbf y}\|_2 \frac{e^{6B\crs}}{1-\delta'}
\left(3+\frac{1}{2B\crs}\ln\frac{1}{1-\delta'}\right)^2 
\nonumber\\&\quad\quad\cdot
B^2 M \left(1+\frac{1}{\sqrt{H}}\right)
\left[1+ \frac{\crs}{\sqrt{m}}
\left(3+\frac{1}{2B\crs}\ln\frac{1}{1-\delta'}\right)
B \frac{e^{3B\crs}}{\sqrt{1-\delta'}}
\right],\label{L_delta}
\end{align}
being valid with a probability $1-H\,e^{-m^2/2}.$

This completes the derivation of $\alpha$, $\beta$ and $L$
for the region ${\mathcal B}\left({\boldsymbol\theta}_0,R_{\delta,\delta'}\right).$ Now we are ready to state our main result for ResNets, which is a consequence of Theorem~\ref{GDthm}.

\section{Main Result and Future Work}
\label{concl}
Our main result in regards to the convergence of the gradient descent algorithm for the residual networks can be summarized as follows.

\begin{theorem}
Consider a data set ${\mathbf x}_i\in{\mathbb R}^d, i=1,\dots,n$
forming the data matrix ${\mathbf X}\in{\mathbb R}^{n\times d}$
with $\|{\mathbf x}_i\|=1,$ for all $i=1,\dots,n$, 
and the label vector ${\mathbf y}\in{\mathbb R}^n.$
Let $\frs({\mathbf x},\boldsymbol{\theta})$ be the ResNet output as defined by \eqref{Res_Net_eq}, such that network width is $m$ and network depth $H.$ Let activation
function $\phi$ satisfies $|\phi'(z)|\leq B$ and 
$|\phi''(z)|\leq M$, along with the property $|\phi({\mathbf a}+{\mathbf b})|\leq |\phi({\mathbf a})|+|\phi({\mathbf b})|.$ Let the data matrix associated with
each layer output associated with the random initialization be denoted as ${\mathbf X}^{(i)}\in{\mathbb R}^{n\times m}, i=1,\dots,H.$
Moreover let half of the entries
of the output layer weight vector ${\mathbf a}$ be 
$\frac{\|{\mathbf y}\|_2}{\sqrt{n}}$ and the other half of the entries be $-\frac{\|{\mathbf y}\|_2}{\sqrt{n}},$ 
and let $\delta$ and $\delta'$ be positive numbers with $0<\delta'<1.$
Define $\alpha_{\delta'}$, $\beta_{\delta'}$, $L_{\delta'}$ 
via the equations
\eqref{target_alpha}, \eqref{beta_delta} and \eqref{L_delta},
respectively. Define $\kappa$ via equation \eqref{kappa}.
Starting from initial weight matrices ${\mathbf W}_0^{(1)},\dots,{\mathbf W}_0^{(H)}$ with i.i.d. ${\mathcal N}(0,1)$ entries, update each weight matrix via gradient descent algorithm 
with step size 
\begin{align}
\eta\leq \frac{1}{\beta_{\delta'}^2}\cdot\min\left(1,
\frac{\alpha_{\delta'}^2}{L_{\delta'} 
\kappa\|{\mathbf y}\|_2}\right)\label{eta}
\end{align}
for the quadratic loss function.
Then, as long as
\begin{gather}
m \geq K_{\delta,\delta'}\, n,\nonumber\\
K_{\delta,\delta'}\triangleq \max\Bigg\{ \frac{64\kappa^2 
B^2 \crs^2 e^{4 B \crs} }{(1-\delta')^4
\left(\ln\frac{1}{1-\delta'}\right)^2
c_{\phi}\lambda({\mathbf X})} ,
\frac{32\kappa^2e^{4B\crs}}{d\delta'^2(1-\delta')^4\cph\lambda({\mathbf X})} \Bigg\}
\end{gather}
holds true, and $H=H(\delta',B,\crs)$ is large enough but does not depend on $m$ or $n$, the iterations satisfy
\begin{align}
\twonorm{\frs(\vct{\theta}_\tau)-\vct{y}}^2\le&\left(1-
\frac{\eta\alpha^2_{\delta'}}{2}\right)^\tau\twonorm{\frs(\vct{\theta}_0)-\vct{y}}^2, \label{iteration}\\
\frac{\alpha_{\delta'}}{4}\twonorm{\vct{\theta}_\tau-\vct{\theta}_0}+
\twonorm{\frs(\vct{\theta}_\tau)-\vct{y}}\le&
\twonorm{\frs(\vct{\theta}_0)-\vct{y}}.
\end{align}
with a probability at least
$1-\kappa_1 m e^{-\kappa_2 m \delta'}-H e^{-m^2/2} -
e^{-\delta^2\frac{n}{2\|{\mathbf X}\|^2}}-
\mathlarger{\sum}_{k=1}^{H-1}e^{-\delta^2\frac{n}{2\|{\mathbf X}^{(k)}\|^2}}$ where $\kappa_1$ and $\kappa_2$ are some constants.
\label{main_result}
\end{theorem}

It follows from \eqref{iteration} that
\begin{align}
\log \twonorm{\frs(\vct{\theta}_\tau)-\vct{y}} &\leq 
\frac{\tau}{2} 
\log\left(1- \frac{\eta\alpha^2_{\delta'}}{2}\right)+
\log \twonorm{\frs(\vct{\theta}_0)-\vct{y}} \nonumber\\
&= \frac{\tau}{2} 
\log\left(1- \frac{\eta\alpha^2_{\delta'}}{2}\right)+ O(\log \sqrt{n}) \label{iteration1}\\
&\leq -\tau \frac{\eta\alpha^2_{\delta'}}{4}+ O(\log \sqrt{n})
\label{iteration2}
\end{align}
where \eqref{iteration1} is a consequence of \eqref{misfit_final},
and \eqref{iteration2} is due to the basic inequality $\log x\leq x-1.$
Therefore if the number of iterations $\tau$ satisfy
\begin{align}
\tau \geq \frac{4}{\eta \alpha^2_{\delta'}} 
O\left(\log \frac{\sqrt{n}}{\epsilon}\right)
\label{iteration3}
\end{align}
then we would have $\twonorm{\frs(\vct{\theta}_\tau)-\vct{y}}\leq \epsilon.$ We have $\alpha_{\delta'}=O(1)$,
$\beta_{\delta'}=O(\sqrt{n})$, 
and $L_{\delta'}=O(\sqrt{n})$
as \eqref{target_alpha},\eqref{beta_delta} and \eqref{L_delta} implies. Hence we conclude from \eqref{eta} and \eqref{iteration3}
that
\begin{align}
\tau=\Omega\left(n^2\log\frac{\sqrt{n}}{\epsilon} \right)    
\end{align}
iterations are sufficient to get $\twonorm{\frs(\vct{\theta}_\tau)-\vct{y}}\leq \epsilon.$

This completes our convergence analysis for the ResNet model.
We see from Theorem~\ref{main_result} that the linear scaling of
the network width with respect to the data size is enough to ensure a
linear convergence to the globally optimal solution for ResNets
when the loss function is quadratic. Future research topics include but are not limited to
\begin{enumerate}[i.]
\item To extend the analysis carried out here to more general loss functions, and to more general network types, such as feedforward networks,
\item To extend the result proved here to non-smooth and common
activation functions, such as ReLU,
\item To extend the regular gradient descent analysis to the stochastic gradient case,
\item To analyze the test loss, and derive a decent generalization bound.
\end{enumerate}

\bibliography{Bibfiles}
\bibliographystyle{unsrt} 


\end{document}